\documentclass[accepted]{uai2022} % for initial submission
\usepackage[american]{babel}
\usepackage{natbib} % has a nice set of citation styles and commands
    \bibpunct{(}{)}{;}{a}{,}{,}
    
\usepackage{mathtools} % amsmath with fixes and additions
\usepackage{booktabs} % commands to create good-looking tables

\usepackage{textcase}

\usepackage{amsmath, amsthm, amssymb, amsfonts, mathtools, graphicx, enumitem}
\usepackage{algorithm,algorithmic}

\usepackage{mkolar_definitions}

\newtheorem*{theorem*}{Theorem}

\usepackage{comment}

\usepackage{tikz}
\usepackage{bbm}
\usetikzlibrary{automata, arrows}
\usetikzlibrary{positioning}

\usepackage[capitalise,nameinlink]{cleveref}
\Crefname{equation}{Eq.}{Eqs.}
\Crefname{assumption}{Assumption}{Assumptions}
\Crefname{condition}{Condition}{Conditions}

\allowdisplaybreaks

\newcommand{\gapmin}{C_\mathrm{gap}}
\newcommand{\gap}{\mathrm{gap}}
\newcommand{\cgap}{C_\mathrm{gap}}
\newcommand{\gapq}{\mathrm{gap}(Q^*)}
\newcommand{\estat}{\varepsilon_{\mathrm{stat},n}}

\hypersetup{
  colorlinks   = true, %Colours links instead of ugly boxes
  urlcolor     = blue, %Colour for external hyperlinks
  linkcolor    = blue, %Colour of internal links
  citecolor   = blue  %Colour of citations
}

\newcount\Comments  % 0 suppresses notes to selves in text
\definecolor{darkred}{rgb}{0.7,0,0}
\definecolor{darkgreen}{rgb}{0,0.5,0}
\definecolor{orange}{rgb}{0.7,0.4,0}
\definecolor{purple}{rgb}{0.8,0.0,0.8}
\title{Offline Reinforcement Learning Under Value and \\ Density-Ratio Realizability: The Power of Gaps}

\author{Jinglin Chen, Nan Jiang}
\affil{%
    Department of Computer Science\\
    University of Illinois Urbana-Champaign\\
    Urbana, IL, USA
    
    \textrm{jinglinc@illinois.edu, nanjiang@illinois.edu}
}
\begin{document}
\maketitle

\begin{abstract}
We consider a challenging theoretical problem in offline reinforcement learning (RL): obtaining sample-efficiency guarantees with a dataset lacking sufficient coverage, under only realizability-type assumptions for the function approximators. While the existing theory has addressed learning under realizability and under non-exploratory data separately, no work has been able to address both simultaneously (except for a concurrent work which we compare in detail). Under an additional gap assumption, we provide guarantees to a simple pessimistic algorithm based on a version space formed by marginalized importance sampling (MIS), and the guarantee only requires the data to cover the optimal policy and the function classes to realize the optimal value and density-ratio functions. While similar gap assumptions have been used in other areas of RL theory, our work is the first to identify the utility and the novel mechanism of gap assumptions in offline RL with weak function approximation.
\end{abstract}

\section{Introduction and Related Works}

%\blfootnote{jinglinc@illinois.edu, nanjiang@illinois.edu}

In offline reinforcement learning (RL), the learner searches for a good policy purely from historical (or offline) data, without direct interactions with the real environment. The lack of  intervention with the system makes offline RL a promising paradigm for learning sequential decision-making strategies in many important real-world applications. 

Early research in offline RL focused on analyzing approximate value and policy iteration algorithms and had significant overlap with the approximate dynamic programming literature \citep{munos2003error, munos2007performance, munos2008finite, antos2008learning,farahmand2010error}. These algorithms and their guarantees typically require relatively strong assumptions on both the expressivity of the function class and the exploratoriness of the dataset. For example, the analyses of Fitted Q-Iteration \citep{ernst2005tree,antos2007fitted} require the function class to be \textit{closed} under Bellman updates (also known as Bellman-completeness), and the offline data distribution to provide coverage (in some technical sense) over \textit{all} candidate policies \citep{chen2019information}. The former requirement is \textit{non-monotone} in the function class, which shatters the standard machine-learning intuition that a richer function class should always have better (or at least no worse) approximation power; it is also closely related to the instability of RL training and the infamous ``deadly triad'' \citep{sutton2018reinforcement, wang2021instabilities}. The latter requirement is very likely violated in practice since we have no control over how the historical data is collected 
\citep{fujimoto2019off}. 

Given these considerations, it is desirable to come up with novel algorithms and/or analyses to relax these assumptions. In particular, the ideal assumption on the function class is \textit{realizability}, that there is a target function of interest (such as the optimal value function) and we only require the function class to (approximately) capture such a function. The ideal assumption on the data distribution is \textit{single-policy} coverage, that it is ok for the data to not cover all policies, as long as an optimal (or sufficiently good) policy is covered. 

In recent years, significant progress has been made towards providing provable guarantees in offline RL under these relaxed assumptions. In particular, the principle of pessimism in face of uncertainty proves to be useful in designing algorithms that work under single-policy coverage \citep{liu2020provably, jin2021pessimism, xie2021bellman,yin2021near,rashidinejad2021bridging}, but most of the existing pessimistic algorithms require Bellman completeness on the function class. On the other hand, relaxing Bellman completeness to realizability has been difficult: there is merely one existing result that requires only the realizability of optimal value function \citep{xie2020batch}, yet their data assumption is even stronger than all-policy coverage. In fact, a recent information-theoretic lower bound by \citet{foster2021offline} confirms that even with a strong notion of all-policy coverage called (all-policy) concentrability, plus the realizability of value functions for \textit{all} policies, offline RL is still fundamentally intractable. 

Despite the lower bound, not all hope is lost. A promising way of breaking the lower bound is to assume the realizability of other functions beyond value functions. %, which is exactly the case in the literature on marginalized importance sampling (MIS). 
Indeed, positive results that are analogues to what we want are established in off-policy evaluation (OPE)---where the goal is to estimate the performance of a target policy from offline data---when additional realizability of \textit{density-ratio} functions is assumed. In particular,   \citet{liu2018breaking, uehara2020minimax} show that, as long as the data covers the target policy, and we are given function classes that can represent \textit{both} the value function and the density-ratio function (or marginalized importance weights) of the target policy, %between the target policy and the data, 
it is possible to estimate the performance of the target policy in a sample-efficient manner. 
One way of using such results for policy learning is to use OPE as a subroutine and optimize a policy using OPE's assessment of the policy's performance. Unfortunately, such a direct application introduces prohibitive expressivity assumptions we wanted to avoid beginning with, such as the realizability of value functions for \textit{all} candidate policies \citep{jiang2020minimax}. 

In this paper, we provide sample-efficiency guarantees for offline RL under the desired assumptions, that data is only guaranteed to cover the optimal policy, and the function classes only represent the optimal value function and density ratio, respectively. Our algorithms have a simple procedure that combines marginalized importance sampling (MIS) with pessimism in a novel fashion. The key enabler of our guarantees is an additional \textit{gap} assumption, that there is a nontrivial gap between the values of the (unique) greedy action and the second-best for every state. Similar gap assumptions are common in RL theory to characterize easy problems in which stronger-than-usual guarantees can be obtained. They are often used to achieve logarithmic or constant regret in bandits and tabular online RL  \citep{bubeck2012regret,ok2018exploration,slivkins2019introduction,lattimore2020bandit,he2021logarithmic,papini2021reinforcement}, and similar guarantees in offline RL under Bellman-completeness and additional structural assumptions on the value-function class \citep{hu2021fast}. They are also used in online RL with function approximation to block exponential error amplification \citep{du2019provably}. To our knowledge, our work is the first one to identify the utility of gap assumptions in offline RL with weak function approximation and offer interesting insights into novel aspects and mechanisms of gaps %, such as disentangling the true gap and that of the function approximator 
(see \pref{sec:appx_error}).

\paragraph{Paper Organization} The rest of the paper is organized as follows: \pref{sec:prelim} introduces preliminary concepts and the problem setting. \pref{sec:alg} describes the algorithm. \pref{sec:main} provides the core analysis. %, with the proofs outlined in \pref{sec:proof}. 
We further extend our results to the setting where the function classes are misspecified (\pref{sec:appx_error}), and when the gap parameter is unknown but we have access to a small amount of online interactions (\pref{sec:unknown_gap}). We conclude the paper with further discussions in \pref{sec:discuss}, including a detailed comparison to the concurrent work of \citet{zhan2022offline} on the same problem.

\section{Preliminaries} \label{sec:prelim}
\paragraph{Markov Decision Processes (MDPs)}
We consider finite horizon episodic MDPs defined in the form of $\Mcal = (\Xcal, \Acal, P, R, H, x_0)$, where $\Xcal=\Xcal_0\bigcup\ldots\bigcup\Xcal_{H-1}$ is the layered state space with $\Xcal_h$ denoting the state space at timestep $h$, $\Acal$ is the action space, $P=(P_0,\ldots,P_{H-1})$ is the transition function with $P_h: \Xcal_h\times\Acal\to\Delta(\Xcal_{h+1})$, $R=(R_0,\ldots,R_{H-1})$ is the reward function with $R_h: \Xcal_h\times\Acal\to [0, 1]$, $H$ is the length of horizon, and $x_0$ is the fixed initial distribution.\footnote{We consider fixed initial state and deterministic reward function. They can be easily generalized to the stochastic case.} We assume the state and action spaces are finite but can be arbitrarily large, and $\Delta(\cdot)$ denotes the probability simplex over a finite set. We define a policy $\pi=\{\pi_0,\ldots,\pi_{H-1}\}$, where for each $h\in[H]$, $\pi_h: \Xcal_h \to \Delta(\Acal)$ is the policy at timestep $h$ and we use $[H]$ to denote $\{0,\ldots,H-1\}$. With a slight abuse of notation, when $\pi_h(\cdot)$ is a deterministic policy, we assume $\pi_h(\cdot):\Xcal_h\to \Acal$. Policy $\pi$ induces a distribution over trajectories from the initial state distribution, which we denote as $\Pr\nolimits_\pi(\cdot)$ and can be described as starting with $x_0$ and $a_h \sim \pi(\cdot|x_h), r_h = R_h(x_h, a_h), x_{h+1} \sim P_h(\cdot|x_h, a_h), \forall h \in[H]$. As a convention, we will use $x_h,a_h,r_h$ to refer the state, action, and reward at timestep $h$ (thus $x_h\in\Xcal_h$). The performance of a policy is measured by its expected return, defined as $v^\pi := \EE_\pi[\sum_{h=0}^{H-1} r_h]$, where the expectation is taken with respect to $\Pr\nolimits_\pi(\cdot)$. For any  $f_h\in\RR^{\Xcal_h\times\Acal}$, we use $\pi_{f_h}(x_h):=\argmax_{a_h\in\Acal}f_h(x_h,a_h)$ to denote its greedy policy at timestep $h$. Among all policies, there always exists a policy, denoted as $\pi^*$, that maximizes the return from all starting states simultaneously. This policy is the greedy policy of the optimal action-value (or Q-) function, $Q^*=(Q_0^*,\ldots,Q_{H-1}^*)$, i.e., $\pi^* = \pi_{Q^*}:=(\pi_{Q^*_0},\ldots,\pi_{Q^*_{H-1}})$. $Q^*$ is the unique solution to the Bellman optimality equations $Q^*_h = \Tcal_h Q^*_{h+1}$, where $\Tcal_h: \RR^{\Xcal_{h+1}\times\Acal} \to \RR^{\Xcal_h\times\Acal}$ is the Bellman optimality operator: $\forall f_{h+1} \in \RR^{\Xcal_{h+1}\times\Acal}$, $(\Tcal_hf_{h+1})(x_h,a_h):=R_h(x_h,a_h)+\EE_{x_{h+1} \sim P_h(\cdot|x_h, a_h)}[ \max_{a_{h+1}}f_{h+1}(x_{h+1}, a_{h+1})]$. We can similarly define policy-specific Q-functions $Q^\pi$ and their state-value function counterparts, namely $V^*$ and $V^\pi$. Another useful concept is the notion of state-action occupancy of a policy $\pi$, $d^\pi_h(x_h',a_h') := \Pr\nolimits_\pi(x_h=x_h',a_h=a_h')$. As a shorthand, we define $d^*_h := d^{\pi^*}_h$ and use $a_{i:j}$ to refer actions $a_i,\ldots,a_j$.

\paragraph{Offline RL} We consider a standard theoretical setup for offline RL, where we are given a dataset $\Dcal=\Dcal_0\bigcup\ldots\bigcup\Dcal_{H-1}$ with the form $\Dcal_h=\{x_h^{(i)},a_h^{(i)},r_h^{(i)},$ $x_{h+1}^{(i)}\}_{i=1}^n$ and $\Dcal_h$ consists of $\{x_h,a_h,r_h,x_{h+1}\}$ tuples sampled i.i.d.~from the following generative process: $(x_h,a_h)\sim d^D_h,r_h=R_h(x_h,a_h),x_{h+1}\sim P_h(\cdot| x_h,a_h)$. Note that $r_h$ and $x_{h+1}$ are generated according to the MDP reward and transition functions, and $d^D_h$ fully determines the quality and coverage of the data distribution. For a given policy $\pi$, $w^\pi_h(x_h,a_h):= d^\pi_h(x_h,a_h) / d^D_h(x_h,a_h)$ measures how well $d^D_h$ covers the occupancy induced by $\pi$ at timestep $h$ and is often known as the density-ratio function or the marginalized importance weight. It plays an important role in offline RL algorithms and analyses. As another shorthand, we use notation $w^\pi=(w_0^\pi,\ldots,w_h^\pi)$ to denote the density-ratio function over all timesteps and notation $w^*:= w^{\pi^*}$ to denote the density ratio of the optimal policy.

\paragraph{Function Approximation} We consider the function approximation setting, where we are given a function class $\Fcal=\Fcal_0\times\ldots\times\Fcal_{H-1}$ with $\Fcal_h\subseteq(\Xcal_h\times\Acal\rightarrow\RR),\forall h\in[H]$ and a weight function class $\Wcal=\Wcal_0\times\ldots\times\Wcal_{H-1}$ with $\Wcal_h \subseteq(\Xcal_h\times\Acal\rightarrow\RR), \forall h\in[H]$. We assume these are finite classes and use $\log(|\Fcal|)$ and $\log(|\Wcal|)$ to measure their statistical capacities. The extension to continuous or infinite classes with a covering argument is standard. By default, for any $f\in\Fcal,$ we assume $f_H=\zero$ for technical simplicity and use $V_f$ to denote its induces state-value function, i.e., $V_f(x_h)=\max_{a_h\in\Acal} f_h(x_h,a_h)$. We will also use $\pi_f(x_h)$ instead of $\pi_{f_h}(x_h)$ for simplicity since only $f_h$ operates on $x_h\in \Xcal_h$ and there is no confusion.

\section{Algorithm} \label{sec:alg}

In this section, we introduce our algorithm PABC (Pessimism under Average Bellman error Constraints), whose pseudo-code is given in \pref{alg:pess_alg}. The algorithm takes two steps: a prescreening step (\pref{line:prescreen}), followed by the main step (\pref{line:pess_select}). We first give an intuition for the main step, deferring the explanation of the prescreening step and the related gap definitions to the later part of this section.

\begin{algorithm}[htb]
	\caption{PABC (Pessimism under Average Bellman error Constraints)\label{alg:pess_alg}}
	\begin{algorithmic}[1]
	    \REQUIRE threshold $\alpha>0$, gap parameter $\cgap$, function class $\Fcal$, weight function class $\Wcal$, and dataset $\Dcal$.
	    \STATE Perform prescreening according to input  $\cgap$: \label{line:prescreen}
	    \begin{align} %\label{line:prescreen}
	    \Fcal(\cgap):=\{f \in \Fcal: \gap(f)\ge\cgap\}.
	    \end{align}\vspace{-1.5em}
		\STATE Find the pessimism value function in  $\Fcal(\cgap)$  subject to average Bellman error constraints \label{line:pess_select}
		\begin{align}
		\label{eq:constraint}
		&\hat f=\argmin_{f\in\Fcal(\cgap)} f_0(x_0,\pi_{f}(x_0))
		\notag\\
    	\text{s.t.}& \max_{w\in\Wcal,h\in[H]} |\Lcal_{\Dcal}(f,w,h)|\le\alpha,
    	\end{align}
    	where the empirical loss $\Lcal_{\Dcal}(f,w,h)$ is defined as
    	\begin{align}
        \Lcal_{\Dcal}(f,w,h) &%&~=\EE_{\Dcal}[w_h(x_h,a_h)(f_h(x_h,a_h)-r_h-f_{h+1}(x_{h+1},\pi_f(x_{h+1})))].\\
        =\frac{1}{n}\sum_{i=1}^n[w_h(x_h^{(i)},a_h^{(i)})(f_h(x_h^{(i)},a_h^{(i)}) \nonumber\\
        &-r_h^{(i)}-f_{h+1}(x_{h+1}^{(i)},\pi_f(x_{h+1}^{(i)})))]. \label{eq:LD}
        \end{align}
    	%\begin{align}
        %\Lcal_{\Dcal}(f,w,h) &
        %=\frac{1}{n}\sum_{i=1}^n[w_h(x_h^{(i)},a_h^{(i)})(f_h(x_h^{(i)},a_h^{(i)})-r_h^{(i)}-f_{h+1}(x_{h+1}^{(i)},\pi_f(x_{h+1}^{(i)})))]. \label{eq:LD}
        %\end{align}
		\ENSURE policy $\pi_{\hat f}$ and return estimation $\hat f_0(x_0, \pi_{\hat f}(x_0))$.
	\end{algorithmic}
\end{algorithm}

The main step (\pref{line:pess_select}) runs a constrained optimization to select a function $\hat f \in \Fcal$, whose greedy policy is the output. The objective of the optimization minimizes the value at the initial state, which is a form of initial-state pessimism \citep{xie2021bellman,zanette2021provable} and proved to be useful in handling insufficient data coverage. The constraints eliminate functions with large \textit{average Bellman errors} \citep{jiang2017contextual, xie2020q}. 

\paragraph{Average Bellman Error Constraints} To provide intuition, we know that $Q^*$ has $0$ average Bellman errors for all state-action pairs, that is, $\forall h\in[H], x_h\in\Xcal_h, a_h\in\Acal$, $(Q^*_h - \Tcal_h Q^*_{h+1})(x_h, a_h) = 0.$  
Thus it also has $0$ average Bellman errors under any distribution $\nu_h$ at timestep $h$:
$$
\EE_{(x_h,a_h) \sim \nu_h}[(Q^*_{h} - \Tcal_h Q^*_{h+1})(x_h,a_h)] = 0.
$$
This holds even if $\nu_h$ is an unnormalized distribution. Therefore, we can safely eliminate any candidate function $f\in\Fcal$, if it has a large average Bellman error $\EE_{\nu_h}[f_h - \Tcal_h f_{h+1}]$ under any (possibly unnormalized) distribution $\nu_h$. Unlike the more standard versions of Bellman errors such as $\EE_{\nu_h}[(f_h - \Tcal_h f_{h+1})^2]$, which squares the Bellman error in each state before taking expectation and cannot be directly estimated due to the infamous \textit{double-sampling} difficulty \citep{baird1995residual, farahmand2011model}, the average Bellman error can be easily estimated. %, since we do not take square or absolute value in each state and directly take a plain expectation. 
In the algorithm, we consider a variety of  (possibly unnormalized) distributions $\nu_h = w_h \cdot d^D_h$ for $w\in\Wcal,h\in[H]$. Since the average Bellman error can only be empirically approximated (see \Cref{eq:LD}), we relax the constraints and allow a threshold $\alpha$ to incorporate the statistical errors. 
We note that the constraint alone is similar to the MABO algorithm by \citet{xie2020q}, but they do not use pessimism and cannot handle insufficient data coverage. They also assume that $\Wcal$ is sufficiently rich to approximate $w^{\pi_f}$ \textit{for all} $f\in\Fcal$, and a main goal of our work is to avoid such ``for all'' assumptions.

\paragraph{Gap and Prescreening} As mentioned in the introduction, a key assumption that enables our results is a gap assumption on value functions. To prepare for the discussion, we define the gap of a function as follows:
\begin{definition}[Gap]
\label{def:gap}
For any $f = (f_0, \ldots, f_{H-1})$, we define its gap at timestep $h \in [H]$ and state $x_h\in\Xcal_h$  as follows: If $\argmax_{a_h \in\Acal} f_h(x_h, a_h)$ is unique, % $a_1,a_2\in\Acal, a_1\neq a_2$ s.t. $f_h(x,a_1)=f_h(x,a_2)=\max_{a\in\Acal}f_h(x,a)$), 
then we define $\gap(f;h,x_h):=\min_{a_h\neq \pi_{f}(x_h)}f_h(x_h,\pi_f(x_h))-f(x_h,a_h)$.  Otherwise, we define $\gap(f;h,x_h):=0$.

The gap of $f$ is then defined as
$$\gap(f):=\min_{h\in[H],x_h\in\Xcal_h}\gap(f;h,x_h).$$
\end{definition}
As we see, this definition of the gap is similar to the one used in prior works \citep{simchowitz2019non,mou2020sample,du2019provably,he2021logarithmic,yang2021q,hu2021fast,wang2021exponential,papini2021reinforcement,wu2021gap}, except that we require a unique optimal action for the gap to be non-zero. A motivating example of similar gap assumptions in other areas of RL theory can be found in \citet{wu2021gap}. 
With \pref{def:gap}, we can now define the minimum gap of a function class:
\begin{definition}[Gap of a function class]
\label{def:gapmin_val_class}
Given a function class $\Gcal=\Gcal_0\times\ldots\times\Gcal_{H-1}$, where $\Gcal_h \subseteq(\Xcal_h\times\Acal\rightarrow\RR), \forall h\in[H]$, we define its gap as
$$\gap(\Gcal):=\min_{g\in\Gcal} \gap(g).$$
\end{definition}

Prior theoretical results relying on similar gap assumptions often make such assumptions on the true optimal value function $Q^*$ \citep{simchowitz2019non, yang2021q}. As we will see in our analyses, however, what is really important for us is that the \textit{learned} function $\hat f$ has a large gap, not the true $Q^*$. Since we have no control over which $f$ in the function class will be finally chosen, we perform the prescreening step in \pref{line:prescreen} to eliminate functions with the gap lower than a pre-defined threshold $\cgap \ge 0$. It is immediate to see that $\gap(\Fcal(\cgap))\ge\cgap$. Of course, this runs into the risk of eliminating $Q^*$, and if we do not want any misspecification, we need to ensure $Q^* \in \Fcal(\cgap)$, which requires that $\cgap \le \gap(Q^*)$. For clarity, in \pref{sec:find_near_optimal} we will assume that we have the knowledge of $\gap(Q^*)$ and can set $\cgap$ accordingly, while later in \pref{sec:unknown_gap} we show how to handle unknown $\gap(Q^*)$. Moreover, as we will see in \pref{sec:appx_error}, when we allow misspecification errors in the analysis, $\gap(Q^*)$ and $\gapmin$ become disentangled, which leads to some interesting implications.

\section{Main Guarantees}  \label{sec:main}
In this section, we present the main sample complexity results of our algorithms. We start with a weak version of guarantee by showing that our algorithm can identify $v^*$, the optimal expected return at the initial state, with polynomial samples under realizability and single-policy coverage assumptions, even without any gap assumption (\pref{sec:find_v_star}). Such a result will also be useful when we handle the unknown gap setting later in \pref{sec:unknown_gap}. Then, \pref{sec:counter} provides an algorithm-specific counterexample to show that our algorithm fails to find a near-optimal policy under these assumptions, motivating the necessity of the gap assumption. Finally, \pref{sec:find_near_optimal} provides the main result of this paper under the additional gap assumption.

\subsection{Estimating the Optimal Expected Return}
\label{sec:find_v_star}
We first show how to identify $v^*$, the optimal expected return of the problem, \textit{without} needing the gap assumption. In this case, we will run \pref{alg:pess_alg} with $\gapmin=0$, that is, without the prescreening step. To our knowledge, there is no prior work that can achieve this goal under our weak assumptions.\footnote{We note that under \citet{zhan2022offline}'s assumptions, their algorithm, with regularization removed, can also identify $v^*$.} Despite not producing a near-optimal policy, this procedure and guarantee allows us to check whether any given policy is close to optimal, assuming we can evaluate the policy's performance by off-policy evaluation or a small amount of online interactions. This capability can be very useful especially in certain model selection scenarios \citep[see e.g.,][Section 5]{modi2020sample}. Indeed, we will reuse this result later in \pref{sec:unknown_gap} to handle the unknown gap setting.

We start by introducing the assumptions. The first two are the standard realizability assumptions.

\begin{assum}[Realizability of $\Fcal$]
\label{assum:realizablity_q}
We assume $Q^*=(Q^*_0,\ldots,Q^*_{H-1})\in\Fcal$.
\end{assum}

\begin{assum}[Realizability of $\Wcal$]
	\label{assum:realizablity_w}
    We assume $w^*=(w^*_0,\ldots,w^*_{H-1})\in\Wcal$.
\end{assum}
%\begin{remark}
We make these assumptions exact for now to allow for a clean presentation of the main results and core proof ideas, and defer the handling of misspecification errors to \pref{sec:appx_error}. Also, following the arguments in \citet{uehara2020minimax, xie2020q}, \pref{assum:realizablity_w} can be further relaxed such that $w^*$ only needs to lie in the convex hull of $\Wcal$.  

Next, we introduce the standard boundedness assumptions. 

\begin{assum}[Boundness of $\Fcal$]
	\label{assum:bound_q}
For any $f\in\Fcal$, we assume $f_h\in (\Xcal_h\times\Acal\rightarrow[0,H-h]),\forall h\in[H]$.
\end{assum}

\begin{assum}[Boundness of $\Wcal$]
	\label{assum:bound_w}
	For any $w\in\Wcal$, we assume $\|w_h\|_\infty \le C, \forall h\in[H]$. 
\end{assum}

\pref{assum:realizablity_w} and \pref{assum:bound_w} together immediately imply that our data covers $\pi^*$: 
\[\frac{d^*_h(x_h,a_h)}{d^D_h(x_h,a_h)}\le C,\forall h\in[H], x_h\in\Xcal_h,a_h\in\Acal.
\]
This version of coverage is often called $\pi^*$-concentrability \citep{scherrer2014approximate, xie2021policy, rashidinejad2021bridging, zhan2022offline}. As we will see when we consider misspecification errors in \pref{sec:appx_error}, we do not really need our data to satisfy $\pi^*$-concentrability, and the definition of coverage can be relaxed using the structure and generalization effects of $\Fcal$ similarly to \citet{jin2021pessimism, xie2021policy}.

With all the above assumptions, we are ready to state our first result formally below. The proof is deferred to \pref{app:proof_find_v_star}. 

\begin{theorem}[Sample complexity of estimating $v^*$]
\label{thm:find_v_star}
Suppose Assumptions \ref{assum:realizablity_q}, \ref{assum:realizablity_w}, \ref{assum:bound_q}, \ref{assum:bound_w} hold and the total samples $nH$ satisfies 
$$
nH\ge \frac{8C^2H^5\log(2|\Fcal||\Wcal|H/\delta)}{\varepsilon^2}.
$$
Then with probability at least $\ge 1-\delta$, running \pref{alg:pess_alg} with $\gapmin=0$ and $\alpha=\varepsilon/(2H)$ guarantees \[|V_{\hat f}(x_0)-v^*|\le \varepsilon.\]
\end{theorem}

\subsection{Algorithm-Specific Counterexample} \label{sec:counter}
Despite being able to identify $v^*$, we show that \pref{alg:pess_alg} cannot be guaranteed to learn a near-optimal policy without further assumptions, even with infinite data. As we will see, a key aspect of the construction is a tie between the values of different actions, so such counterexamples can be effectively excluded by assuming a unique optimal action. 

The counterexample is given in \pref{fig:counter}. 
Circles denote states and arrows denote actions with deterministic transitions, and states without arrows have a default $\mathrm{null}$ action. There is only a $+1$ reward at state $x_C$, while the rewards are $0$ everywhere else. Taking $\textrm{L}_1$ at $x_0$ deterministically transits to $x_A$ and we omit the remaining specifications as they are clearly indicated in the figure.

\begin{figure}
	\centering
	\includegraphics[scale=1]{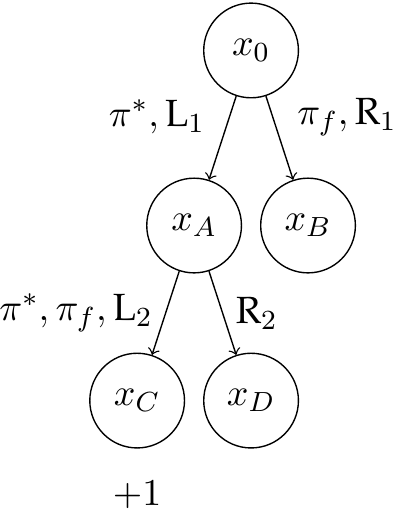}
	\caption{Algorithm-specific counterexample without the gap assumption.}
	\label{fig:counter}
\end{figure}

It is easy to see that the optimal policy $\pi^*$ takes action $\textrm{L}_1$ at state $x_0$ and action $\textrm{L}_2$ at state $x_A$. We construct a bad function $f$, which only differs from $Q^*$ at $(x_0,\textrm{R}_1)$ and $(x_B,\mathrm{null})$ by setting $f_0(x_0,\textrm{R}_1)=1$ and $f_1(x_B,\mathrm{null})=1$. By adversarial tie breaking, we assume $\pi_f(x_0)=\textrm{R}_1$. We immediately have a realizable class $\Fcal=\{Q^*,f\}$. It is easy to see that $\pi_f$ is not $\varepsilon$-optimal for any $\varepsilon<1$ because it deviates from the optimal branch at $x_0$. In addition, we let data $d^D$ covers $(x_0,\textrm{L}_1)$, $(x_A,\textrm{L}_2)$, $(x_C,\mathrm{null})$. For the weight function, we define an \textit{invalid} weight function $w_{\textrm{bad}}$ that puts all weight on $(x_0,\textrm{R}_1),(x_A,\textrm{L}_2),(x_C,\mathrm{null})$ in each level respectively. Then we also have a realizable class $\Wcal=\{w^*,w_{\textrm{bad}}\}$. 

Since both $f$ and $Q^*$ have zero population average Bellman error under $\Wcal$, and $f_0(x_0, \pi_f(x_0)) = Q^*_0(x_0, \pi_{Q^*}(x_0)) = 1$, either of them can be the $\hat f$ learned in \pref{alg:pess_alg}, but returning $\pi_f$ leads to failure of learning. %This implies that both of them satisfy all constraints in \pref{alg:pess_alg}, and they have equal probabilities to be returned. Therefore $f$ (or equivalently $\pi_f$) will be returned with 0.5 probability, and our algorithm fails. 
We note that the reason for failure is that no data covers the state $\pi_f$ visits.

\paragraph{Additional Consistency Constraints} In our setting, there are additional constraints that one can add to ensure some form of consistency. 
For example, for any $f\in\Fcal$, we can additionally require that there exists $w\in\Wcal$ that is consistent with $\pi_{f}$, since $w^* \in \Wcal$ should only give non-zero weight to actions chosen by $\pi_f$ (i.e., $\forall h\in[H], w_h(x_h,a_h)=0 \text{ if } a_h\neq \pi_f(x_h)$). In addition, as we can estimate $v^*$ with the assumptions in \pref{thm:find_v_star} and we know $\EE_{d^D}[w^* \cdot R] = v^*$, we can eliminate any $w \in \Wcal$ that violates this condition. While these constraints are reasonable (or at least harmless) and may be of independent interest, we can verify that they do not help with this counterexample, which implies our algorithm fails even under these additional consistency checks.

\subsection{Learning a Near-Optimal Policy}
\label{sec:find_near_optimal}
As mentioned above, a key aspect of the counterexample is a tie between the values of actions. In this section, we show that a positive gap assumption not only excludes the counterexample, but enables a general guarantee for learning near-optimal policies with our \pref{alg:pess_alg}. 
\begin{assum}[Gap of $Q^*$]
\label{assum:gap_plus}
The gap of $Q^*$ satisfies
\[\gap(Q^*)>0.\]
%\[\gap(Q^*)\ge \gapmin.\]
\end{assum}
Here the implicit assumption is that we want $\gap(Q^*)$ to be sufficiently large, as our later sample complexity guarantees will scale inversely with $\gap(Q^*)$. Note that
\pref{assum:gap_plus} is stronger than the standard gap assumption in the literature \citep{simchowitz2019non, yang2021q, hu2021fast}. Compared with their definition, we additionally assume the optimal action is unique at each state. On the other hand, these papers require additional strong assumptions (e.g., linear MDPs, Bellman-completeness, or pointwise convergence) or focus on the tabular setting, whereas we handle general function approximation in offline RL under weak realizability-type assumptions. Plus, the technical mechanisms under which the gap plays a role in the analyses are very different, so the assumptions are not really comparable.

For now, we assume $\gapq$ is known and will later handle the case of an unknown gap \pref{sec:unknown_gap}. Plus, as a side effect of handling misspecification errors,   \pref{sec:appx_error} will lift the stringent gap assumption in a novel and interesting manner. 

We now state the guarantee of learning a near-optimal policy under the gap assumption. A sketch of proof is provided after the theorem statement, while the complete proof is deferred to \pref{app:proof_main}.

\begin{theorem}[Sample complexity of learning a near-optimal policy]
\label{thm:main}
	Suppose Assumptions \ref{assum:realizablity_q}, \ref{assum:realizablity_w}, \ref{assum:bound_q}, \ref{assum:bound_w}, \ref{assum:gap_plus} hold and the total number of samples $nH$ satisfies 
	\[nH\ge \frac{8C^2H^7\log(2|\Fcal||\Wcal|H/\delta)}{\varepsilon^2 \gapq^2}.
	\]
	Then with probability at least $\ge 1-\delta$, running \pref{alg:pess_alg} with $\alpha=\varepsilon\gapq/(2H^2)$ and $\gapmin=\gapq$ guarantees \[v^{\pi_{\hat f}} \ge v^*-\varepsilon.\]
\end{theorem}

\paragraph{Proof sketch of \pref{thm:main}} 
As standard, all our results depend on a high-probability concentration event, that
%\begin{align*}
$\abr{\Lcal_{\Dcal}(f,w,h)-\EE[\Lcal_{\Dcal}(f,w,h)]}\le \estat$ holds for all $f\in\Fcal, w\in\Wcal, h\in[H]$ with high probability; 
%\end{align*}
%\end{lemma}
the detailed expression of $\estat$ is given in \pref{lem:conc}. From \pref{lem:conc}, for any $f\in\Fcal$ that satisfies all constraints in \pref{alg:pess_alg}, we can guarantee the population loss to be small, that is, 
%\[
$\abr{\EE\sbr{\Lcal_{\Dcal}(f,w,h)}}\le \estat+\alpha.$ 
%\]

The central step of our proof is to use a telescoping argument and the gap assumption to establish the following inequality:
\begin{align*}
&V_0^*(x_0)\ge V_{\hat f}(x_0)\ge V_0^*(x_0)-H(\estat+\alpha)
\\
&\quad+\gapq \EE\left[\sum_{h=0}^{H-1}\one\{\pi_{\hat f}(x_h)\neq \pi^*(x_h)\}\mid \pi^*\right].
\end{align*}
%\begin{align*}
%&V_0^*(x_0)\ge V_{\hat f}(x_0)\ge V_0^*(x_0)-H(\estat+\alpha)+\gapq \EE\left[\sum_{h=0}^{H-1}\one\{\pi_{\hat f}(x_h)\neq \pi^*(x_h)\}\mid \pi^*\right].
%\end{align*}
This implies the policy deviation can be bounded as
\[
\EE\left[\sum_{h=0}^{H-1}\one\{\pi_{\hat f}(x_h)\neq \pi^*(x_h)\}\mid\pi^*\right]\le \frac{H(\estat+\alpha)}{\gapq}.
\]
The LHS of this inequality is the probability that the learned policy $\pi_{\hat f}$ disagrees with the optimal policy $\pi^*$, along the distribution induced by $\pi^*$. From here, we can apply the RL-to-supervised-learning (SL) reduction in imitation learning (e.g., Theorem 2.1 in \citet{ross2010efficient}) to translate it to the final performance difference bound between $\pi_{\hat f}$ and $\pi^*$. We also provide a different proof in \pref{app:proof_main}, which itself may be of independent interest.

\section{Robustness to Misspecification} 
\label{sec:appx_error}
We now consider the case when $Q^*$ and $w^*$ may not exactly belong to $\Fcal$ and $\Wcal$, but can be reasonably approximated up to small errors. More often than not, such robustness results in RL theory are nothing but routine exercises where the proofs are largely straightforward extensions of those for the exact case. In our case, however, the misspecification analyses reveal an interesting phenomenon of disentangling the true gap of $Q^*$ and that of $\Fcal$, and how our gap and coverage assumptions can be relaxed in nontrivial ways. 

We start with defining the approximation errors of our function classes. Inspired by \citet{xie2020q}, we define the approximation error of $\Wcal$ as 
\begin{align}
\label{eq:appx_w}
%\varepsilon_{\Wcal}=~\min_{w\in\Wcal}\max_{f\in\Fcal}\max_{h\in[H]}&|\EE_{d^D_h}[w_h\cdot (f_h-\Tcal_h f_{h+1})] -\EE_{d^*_h}[f_h-\Tcal_h f_{h+1}]|
\varepsilon_{\Wcal}=~\min_{w\in\Wcal}\max_{f\in\Fcal}\max_{h\in[H]}&|\EE_{d^D_h}[w_h\cdot (f_h-\Tcal_h f_{h+1})]\notag
\\
&-\EE_{d^*_h}[f_h-\Tcal_h f_{h+1}]|
\end{align}
and use $\tilde w^*$ to denote the best approximator in $\Wcal$ that obtains the minimum. 
The expression inside the min-max-max measures the difference between $d_h^D \cdot w_h$ and $d_h^*$, using $f_h - \Tcal_h f_{h+1}$ for $f\in\Fcal$ as discriminators. When $w^* \in \Wcal$ (\pref{assum:realizablity_w}), $d_h^D \cdot w_h^* = d_h^*$ (because $w_h^*$ is defined as $d_h^*/d_h^D$), so we have $\varepsilon_{\Wcal} = 0$. However, the opposite direction is \textit{not} always true: 
%A key property of this definition is that $w^*\in\Wcal \Rightarrow \varepsilon_{\Wcal} = 0$, but the opposite is \textit{not} true. 
it is entirely possible to achieve $\varepsilon_{\Wcal} = 0$ when $w^* \notin \Wcal$ (or even when $d_h^*(x_h,a_h)/d_h^D(x_h,a_h) = \infty$ for some $(x_h,a_h)$ and $w^*$ does not exist), as long as $d_h^D \cdot \tilde w^*_h$ and $d^*_h$ can not be distinguished by $f_h - \Tcal_h f_{h+1}$ for $f\in\Fcal$ as \emph{discriminators}.\footnote{The idea of using discriminators has also been explored in \citet{farahmand2017value,sun2019model,modi2020sample,modi2021model}, but the application is different here.} We also provide an example in \pref{app:conc_example}. Note that since our data coverage assumption is implicitly made through the realizability and boundedness of $\Wcal$ (see the discussion below \pref{assum:bound_w}), this means that our data coverage assumption is also relaxed using the information of $\Fcal$, which is a common characteristics of recent results in offline RL (e.g., \citet{xie2021bellman} also use the Bellman error class induced by the value function class as discriminators, which is similar to our definition at a high level), but not enjoyed by the concurrent work of \citet{zhan2022offline}.

For function class $\Fcal$, we define the approximation error in a way that uses $\Wcal$ as discriminators, plus a term that measures the difference under the initial state $x_0$:
\begin{align}
\label{eq:appx_f}
\varepsilon_{\Fcal}=&\min_{f\in\Fcal}\max_{w\in\Wcal}\max_{h\in[H]}(|\EE_{d^D_h}[w_h\cdot (f_h -\Tcal_h f_{h+1})]|
\notag \\
&+\abr{f_0(x_0,\pi_f(x_0))-Q^*_0(x_0,\pi^*(x_0))})
%\varepsilon_{\Fcal}=&\min_{f\in\Fcal}\max_{w\in\Wcal}\max_{h\in[H]}(|\EE_{d^D_h}[w_h\cdot (f_h -\Tcal_h f_{h+1})]|+\abr{f_0(x_0,\pi_f(x_0))-Q^*_0(x_0,\pi^*(x_0))})
\end{align}
%\[\tilde Q^*_\Wcal =\argmin_{f\in\Fcal}\max_{w\in\Wcal}\abr{\EE_{d^D}[w\cdot (f -\Tcal f)]}.\]
and use $\tilde Q^*_\Fcal$ to denote the best approximator that achieves the minimum value.
Under mild regularity assumptions on $\Wcal$,\footnote{Namely, $w \ge 0$ and $\EE_{d^D}[w] = 1$. The former is trivial and the latter can be easily verified approximately on data.} it is straightforward to show that $\varepsilon_{\Fcal}$ is weaker than $\ell_\infty$ error up to multiplicative constants:
\begin{align}\label{eq:compare_epsF}
\varepsilon_{\Fcal}\le 3\min_{f\in\Fcal}\max_{h\in[H]}\|f_h-Q_h^*\|_{\infty}, 
\end{align}
and a more detailed discussion can be found in \pref{lem:appx_f_vs_infty}. 
Similarly, we can define the function class $\Fcal(\gapmin)$ related approximation error $\varepsilon_{\Fcal(\gapmin)}$ and its best approximator $\tilde Q^*_{\Fcal(\gapmin)}$ by replacing $\Fcal$ with $\Fcal(\gapmin)$ in \Cref{eq:appx_f}.

\subsection{Estimating the Optimal Expected Return}
With the approximation error defined above, we now extend \pref{thm:find_v_star} to the approximate case. Assuming the approximation error $\varepsilon_\Fcal$ (or a reasonably tight upper bound of it) is known, we can relax the constraint to ensure that $\tilde Q^*_{\Fcal(\gapmin)}$ is not eliminated and obtain the sample complexity guarantee as in \pref{thm:find_v_star_appx}. The full proof is deferred to \pref{app:proof_find_v_star_approx}. As before, we do not need the gap assumption to identify $v^*$ approximately and can run the algorithm with $\gapmin = 0$. 

\begin{theorem}[Robust version of \pref{thm:find_v_star}]
\label{thm:find_v_star_appx}
Suppose Assumptions \ref{assum:bound_q}, \ref{assum:bound_w} hold and the total number of samples $nH$ satisfies 
\[
nH\ge \frac{8C^2H^5\log(2|\Fcal||\Wcal|H/\delta)}{\varepsilon ^2}.
\]
Then with probability $1-\delta$, running \pref{alg:pess_alg} with $\alpha=\varepsilon /(2H)+\varepsilon_{\Fcal}$ and $\gapmin=0$ guarantees \[|V_{\hat f}(x_0)-v^*| \le  \varepsilon +H\varepsilon_{\Fcal}+H\varepsilon_\Wcal.\]
\end{theorem}

While we need the knowledge of $\varepsilon_\Fcal$ to set $\alpha$, we do not need to know $\varepsilon_\Wcal$, which shows a difference between the behaviors of $\Fcal$ and $\Wcal$. This is also the case in the next subsection where we try to learn a near-optimal policy.

\subsection{Learning a Near-Optimal Policy}
Similarly, we can also extend \pref{thm:main} to the misspecified case. Our guarantee is established with a user-specified $\gapmin$ parameter and the approximation error related to prescreened class $\Fcal(\gapmin)$. We provide the sample complexity guarantee in \pref{thm:main_appx} and the complete proof in \pref{app:proof_main_approx}.
\begin{theorem}[Robust version of \pref{thm:main}]
\label{thm:main_appx}
Suppose Assumptions \ref{assum:bound_q}, \ref{assum:bound_w} hold and the total number of samples $nH$ satisfies 
\[nH\ge \frac{8C^2H^7\log(2|\Fcal||\Wcal|H/\delta)}{\varepsilon^2 \gapmin^2}.\]
Then with probability $1-\delta$, running \pref{alg:pess_alg} with a user-specified $\gapmin$ and $\alpha=\varepsilon \gapmin/(2H^2)+\varepsilon_{\Fcal(\gapmin)}$ guarantees \[v^{\pi_{\hat f}} \ge v^*-\varepsilon  - \frac{(H^2+H)\varepsilon_{\Fcal(\gapmin)}+H^2\varepsilon_\Wcal}{\gapmin}.\]
\end{theorem}

\pref{thm:main_appx} gives us a convenient way to set the gap parameter $\gapmin$. We also provide a sample complexity guarantee (\pref{corr:main_appx}) in \pref{sec:corr_main_appx} for the case that $\gapq$ and the $\ell_\infty$ approximation error of $\Fcal$ are known.

\paragraph{Unknown Approximation Errors} Notice that in the robustness results  (\pref{thm:find_v_star_appx} and \pref{thm:main_appx}) we require the knowledge of approximation errors $\varepsilon_\Fcal$ or $\varepsilon_{\Fcal(\gapmin)}$ to set the threshold $\alpha$ in PABC (\pref{alg:pess_alg}). In \pref{app:lang} we show a variant of PABC based on Lagrangians (PABC-L; \pref{alg:pess_lang_alg}) does not require such knowledge, and still enjoys the same sample complexity guarantees. In PABC-L, the original constraints in \Cref{eq:constraint} are moved to the objective, thus the threshold $\alpha$ is no longer needed as the input. We refer the reader to \pref{app:lang} for the formal description of PABC-L and its results and proofs. 

\paragraph{Relaxed Gap Assumption} An outstanding characteristic of \pref{thm:main_appx} is that it no longer depends on $\gap(Q^*)$ explicitly, and only depends on $\gapmin$, a parameter of our choice. Therefore, it may seem to have lifted \pref{assum:gap_plus} that $\gap(Q^*) > 0$, as we can choose $\gapmin$ to be sufficiently large. However, below we show that this issue is more complicated than it may seem, and while our result does relax \pref{assum:gap_plus} in significant ways, it does so in a very nuanced manner. 

First of all, in the worst-case scenario, \pref{assum:gap_plus} is still needed to provide non-vacuous guarantees. This is because, if $Q^*$ has no gap, yet we artificially create a large $\gapmin$ in our prescreened function class $\Fcal(\gapmin)$, we could eliminate all the good approximations of $Q^*$. Among the remaining functions, the best $\ell_\infty$ approximation of $Q^*$ must have an $\ell_\infty$ error no less than $\gapmin$, and if we plug that into the $\varepsilon_{\Fcal(\gapmin)}$ term in the approximation guarantee, the $\gapmin$ on the numerator and the denominator will cancel out, leaving a constant suboptimality gap which makes the guarantee vacuous.

Having said that, the nuance here is that we do \textit{not} use the most stringent $\ell_\infty$ norm to define $\varepsilon_{\Fcal(\gapmin)}$, but rather use an average notion of error (\Cref{eq:appx_f}), which is possibly much smaller than the $\ell_\infty$ error (\Cref{eq:compare_epsF}). Therefore, there are still cases where $\gap(Q^*) = 0$ yet our result yields nontrivial guarantees. As a concrete example, imagine a $Q^*$ that has large gaps in most states, but the gap is $0$ in a few ``bad'' states. In this case, $\gap(Q^*)$ is $0$. However, there can still exist $\tilde Q_{\Fcal(\gapmin)}^*$ that approximates $Q^*$ well everywhere except on those bad states, and as long as no $w\in\Wcal$ puts significant probabilities on the bad states, we  have $\varepsilon_{\Fcal(\gapmin)} \ll \gapmin$ and hence \pref{thm:main_appx} will provide meaningful guarantees.

\section{Handling the Unknown Gap Parameter with Online Access} \label{sec:unknown_gap}
In this section, we extend the main algorithm and analyses in \pref{sec:main} in a different direction than \pref{sec:appx_error}. In particular, we are concerned about the fact that \pref{thm:main} assumes the knowledge of $\gap(Q^*)$. While it is common for offline RL algorithms to have hyperparameters that need to be tuned separately (and this is particularly the case for version-space-based algorithms \citep{jiang2017contextual, xie2021bellman}), here we show that we can address the unknown $\gap(Q^*)$ issue by a small amount of additional online interactions for Monte-Carlo policy evaluation. This is particularly interesting as our result provides an example of how one can use a small amount of online interactions to mitigate limitations in purely offline learning, a practically relevant problem that is also of great interest to the RL theory community \citep{xie2021policy}.

\begin{algorithm}[htb]
	\caption{PABC-OA (PABC with Online Access)	\label{alg:unknown_gap}}
	\begin{algorithmic}[1]
	    \STATE \textbf{Input}: function class $\Fcal$, weight function class $\Wcal$, and dataset $\Dcal$ (with size $|\Dcal_h|=n,\forall h\in[H]$).
		\FOR{$t=0,1,\ldots$} %\label{lin:forloop}
		\STATE Set $\mathrm{gap}^{\mathrm{guess}}_t=H/2^t$.
		\STATE Use $n$ and $\mathrm{gap}^{\mathrm{guess}}_t$ to calculate $\varepsilon_t=\sqrt{8C^2H^6\iota(t)/(n(\mathrm{gap}^{\mathrm{guess}}_t)^2)}$, where $\iota(t)=\log(24|\Fcal||\Wcal|H\cdot 2^t/\delta)$.
		\STATE Run \pref{alg:pess_alg} with $\alpha=\varepsilon_t/(2H)$ and get scalar estimation $\hat v^*_t$. \label{line:find_v}
		\STATE Run \pref{alg:pess_alg} with $\alpha=\varepsilon_t\mathrm{gap}^{\mathrm{guess}}_t/(2H^2)$ and $\gapmin=\mathrm{gap}^{\mathrm{guess}}_t$, and get policy $\hat \pi_t$. \label{line:main}
		\STATE Estimate $v^{\hat \pi_t}$ by running Monte Carlo algorithm with $\tilde O(H^3\log(1/\delta)/\varepsilon_t^2)$ online samples and denote the estimate as $\hat v^{\hat \pi_t}$. \label{line:mc}
		\IF {$\hat v^{\hat \pi_t}\ge \hat v^*_t- 3\varepsilon_t$} \label{line:cond}
		\STATE Output $\hat \pi_t$ and terminate. \label{line:check}
		\ENDIF
		\ENDFOR
	\end{algorithmic}
\end{algorithm}%PABC-OA (PABC with Online Access)
As shown in \pref{alg:unknown_gap}, the algorithm PABC-OA (PABC with Online Access) proceeds iteration by iteration. We start with the maximum possible value of the unknown $\gapq$. For simplicity, we choose $H$ here, and alternatively we can also use $\max_{f\in\Fcal} \gap(f)$ which is tighter. In iteration $t$, we use $\mathrm{gap}^\mathrm{guess}_t=H/2^t$ as the guess of $\gapq$ and calculate the desired $\alpha$ according to \pref{thm:find_v_star} to estimate $v^*$ (\pref{line:find_v}), or calculate the desired $\alpha$ and $\gapmin$ according to \pref{thm:main} to find a near-optimal policy (\pref{line:main}). Finally we conduct Monte-Carlo policy evaluation with online samples (\pref{line:mc}). If the stopping condition (\pref{line:cond}) is satisfied, we are guaranteed to learn a near-optimal policy and can terminate (\pref{line:check}). Otherwise, we proceed to the next iteration, shrink our guessed value of $\mathrm{gap}^\mathrm{guess}_t$, and continue the routine. We can observe an interesting connection between \pref{thm:find_v_star} and \pref{thm:main}, and identifying $v^*$ is indeed useful.

It can be shown that \pref{alg:unknown_gap} will terminate once the guessed value $\mathrm{gap}^{\mathrm{guess}}_t=H/2^t$ drops below the true value $\gapq$, which leads to the sample complexity result in \pref{thm:main_unknown}. The formal proof can be found in \pref{app:proof_main_unknown_gap}.

\begin{theorem}[Sample complexity of learning a near-optimal policy with unknown $\gapq$]
\label{thm:main_unknown}
Suppose Assumptions \ref{assum:realizablity_q}, \ref{assum:realizablity_w}, \ref{assum:bound_q}, \ref{assum:bound_w}, \ref{assum:gap_plus} hold but $\gapq$ is unknown. Assume we have a dataset $\Dcal$ with size $n$ for each $\Dcal_h$ and additional online access to collect \[\tilde O\rbr{\frac{n\log(1/\delta)}{C^2H}}\] samples. 
Then with probability at least $1-\delta$, the output policy $\hat \pi$ from \pref{alg:unknown_gap} satisfies 
\begin{align}
\label{eq:unknonw_q_accu}
v^{\hat \pi}\ge v^* - 5\sqrt{\frac{32C^2H^6\iota(\log(2H/\gapq))}{n\gapq^2}},    
\end{align}
where $\iota(t)=\log(24|\Fcal||\Wcal|H\cdot 2^t/\delta)$.
\end{theorem}
%\begin{remark}
The suboptimality in \Cref{eq:unknonw_q_accu} has the same order (up to polylog terms) as that of running \pref{alg:pess_alg} with known $\gapq$ in \pref{thm:main}. 
If we set this value to be $\varepsilon'$, i.e., $\varepsilon':=5\sqrt{\frac{32C^2H^6\iota(\log(2H/\gapq))}{n\gapq^2}}$, then the number of required online samples is $\tilde O\rbr{\frac{H^5\log(1/\delta)}{(\varepsilon'\gapq)^2}}$, which does not depend on the complexity of the function classes $\Fcal$ and $\Wcal$.

\section{Discussion and Conclusion} \label{sec:discuss}
We conclude the paper with a detailed discussion of how our work compares to the closely related concurrent work of \citet{zhan2022offline}, which also provides a good summary of our contributions and promising future directions.

The very recent work of \citet{zhan2022offline} aims at solving the same problem:\footnote{Their results are in the discounted setting whereas ours in the finite horizon setting, but this is a superficial difference and translating each of the results into the other setting is not difficult.} offline RL under only single-policy coverage and realizability assumptions. Similar to our counterexample in \pref{sec:counter}, they also realize the difficulties in the setting where the optimal weight and value functions are realizable in a straightforward manner. Instead of making a gap assumption like we do, they attack the problem from a different angle by introducing regularization into the Lagrangian of the linear program for MDPs. 

Despite that the two approaches have some fundamental differences (which we will elaborate further below), it is still worth comparing the nature of the two results. To this end, our approach has several advantages:
\begin{enumerate}[leftmargin=*]
\item Regularization changes the definition of the value function in \citet{zhan2022offline}. In fact, the function they need to realize does not obey any form of Bellman equations, and probably should not be called value functions anymore. This makes their realizability assumption somewhat difficult to interpret and connect to the existing literature. In contrast, we work with the most standard notion of $Q^*$. 
\item Due to regularization, the policy learned by \citet{zhan2022offline} is generally suboptimal even with infinite data, so the strength of regularization needs to be carefully controlled for the bias-variance trade-off. As a result, when competing with $\pi^*$, their sample complexity rate is $O(1/\varepsilon^6)$, which is much slower than our $O(1/\varepsilon^2)$. 
\item Our coverage assumption can be significantly relaxed using the structure of $\Fcal$; see discussion in \pref{sec:appx_error}.  
While this is standard in recent offline RL works based on Bellman-completeness assumptions \citep{jin2021pessimism, xie2021bellman}, \citet{zhan2022offline}'s guarantee relies on the boundedness of the raw density ratios and does not enjoy such a relaxation.  
\end{enumerate}
That said, \citet{zhan2022offline}'s result is also attractive in several aspects:
\begin{enumerate}[leftmargin=*]
\item They do not require gap assumptions. While similar gap assumptions are standard in RL theory literature, it is unclear how prevalent it is in real problems and how algorithms that depend on gap assumptions perform in problems when the assumption is violated.
\item Our guarantees only hold if the data covers $\pi^*$ (though the notion of coverage can be relaxed using a structure of $\Fcal$, as mentioned above). In comparison, \citet{zhan2022offline} can still provide meaningful guarantees even when $\pi^*$ is not covered by data, in which case they compete with the best policy under data coverage. 
\item Regarding computation, their algorithm is a convex-concave minimax optimization problem when the function classes are convex. In comparison, the computational characteristics of our method are less clear, though we note that a Lagrangian form of our main step (\pref{line:pess_select}) (see \pref{app:lang} for details) is similar to the kind of minimax optimization commonly found in the MIS literature \citep{nachum2019dualdice,uehara2020minimax, yang2020off, jiang2020minimax}. 
\end{enumerate}
 
We reiterate that these comparisons are made only on the results themselves. The two works take fundamentally different approaches and are of independent interests. For example, despite that both works use density-ratio functions, \citet{zhan2022offline}'s method is based on the linear programming (LP)-formulation of MDPs where the optimal state-value function $V^*$ is modeled, whereas we model the optimal Q-function $Q^*$. This difference is more significant than it may seem, as the LP formulation and the Bellman optimality equations for $Q^*$ are very different foundations for designing learning algorithms, and the gap assumption only makes sense for Q-functions and cannot be used in state-value functions. That said, it will be interesting to investigate if the two works can borrow each other's ideas to address their own weaknesses, which we leave to future investigation.

\begin{acknowledgements} 
NJ acknowledges funding support from ARL Cooperative Agreement W911NF-17-2-0196, NSF IIS-2112471, NSF CAREER award, and Adobe Data Science Research Award. 
\end{acknowledgements}

%\clearpage

\bibliography{refs}

\onecolumn
\appendix

\section{Proof of Main Results}
In this section, we provide the complete proofs of our main results in \pref{sec:main}. We start with some helper lemmas in \pref{app:helper_lemma_main}. Then we show the proof of \pref{thm:find_v_star} in \pref{app:proof_find_v_star}. Finally, we provide the proof of \pref{thm:main} in \pref{app:proof_main}.

\subsection{Helper Lemmas}
\label{app:helper_lemma_main}
\begin{lemma}[Concentration]
\label{lem:conc}
With probability at least $1-\delta$, for any $f\in\Fcal,w\in\Wcal,h\in[H]$ we have,
\[
\abr{\Lcal_{\Dcal}(f,w,h)-\EE[\Lcal_{\Dcal}(f,w,h)]}\le 2CH\sqrt{\frac{\log(2|\Fcal||\Wcal|H/\delta)}{2n}}=:\estat.
\]
\end{lemma}
\paragraph{Remark}
Here we apply Hoeffding's inequality to show the concentration result. Similar as \citet{xie2020q}, we can also apply Bernstein's inequality, but the dominating rate would be the same.
%\end{remark}
\begin{proof}
Firstly, we fix
$f\in\Fcal,w\in\Wcal,h\in[H]$. From the boundedness assumptions (\pref{assum:bound_q} and \pref{assum:bound_w}), for any sample $(x_h^{(i)},a_h^{(i)},r_h^{(i)},x_{h+1}^{(i)})$ in the dataset, we have
\begin{align*}
\abr{w_h(x_h^{(i)},a_h^{(i)})(f_h(x_h^{(i)},a_h^{(i)})-r_h^{(i)}- f_h(x_{h+1}^{(i)},\pi_f(x_{h+1}^{(i)})))}\le CH.
\end{align*}
Then since our dataset is i.i.d., applying Hoeffding's inequality yields that with probability at least $1-\delta/(|\Fcal||\Wcal|H)$,
\begin{align*}
\abr{\Lcal_{\Dcal}(f,w,h)-\EE[\Lcal_{\Dcal}(f,w,h)]}\le 2CH\sqrt{\frac{\log(2|\Fcal||\Wcal|H/\delta)}{2n}}.
\end{align*}
Finally, union bounding over $f\in\Fcal,w\in\Wcal,h\in[H]$ gives us that with probability at least $1-\delta$, for any $f\in\Fcal,w\in\Wcal,h\in[H]$,
\[
\abr{\Lcal_{\Dcal}(f,w,h)-\EE[\Lcal_{\Dcal}(f,w,h)]}\le 2CH\sqrt{\frac{\log(2|\Fcal||\Wcal|H/\delta)}{2n}}:=\estat.
\]
This completes the proof.
\end{proof}

\begin{lemma}[Population loss and average Bellman error]
	\label{lem:trans}
For any $f\in\Fcal,w\in\Wcal,h\in[H]$, we have
\[
\EE[\Lcal_{\Dcal}(f,w,h)]=\EE_{(x_h,a_h)\sim d^D_h}[w_h(x_h,a_h)(f_h(x_h,a_h)- (\Tcal_h f_{h+1})(x_h,a_h)))]
	\]
and
\[
\EE[\Lcal_{\Dcal}(f,w^*,h)]=\Ecal(f,\pi^*,h)=\EE[f_h(x_h,a_h)-R_h(x_h,a_h)-f_{h+1}(x_{h+1},a_{h+1})\mid a_{0:h}\sim\pi^*, a_{h+1}\sim \pi_f],
	\]
where $\Ecal(\cdot)$ is the Q-type average Bellman error \citep{jin2021bellman,du2021bilinear} \[\Ecal(f,\pi,h)=\EE[f_h(x_h,a_h)-R_h(x_h,a_h)-f_{h+1}(x_{h+1},a_{h+1})\mid a_{0:h}\sim\pi, a_{h+1}\sim \pi_f]. 	\]
\end{lemma}
\begin{proof}
These equations can be simply shown from the data generating process and the definition of population loss and empirical loss. For any $f\in\Fcal,w\in\Wcal,h\in[H]$, we have
\begin{align*}
	&~\EE[\Lcal_{\Dcal}(f,w,h)]\\
	=&~\EE\sbr{\frac{1}{n}\sum_{i=1}^n[w_h(x_h^{(i)},a_h^{(i)})(f_h(x_h^{(i)},a_h^{(i)})-r_h^{(i)}-f_{h+1}(x_{h+1}^{(i)},\pi_f(x_{h+1}^{(i)})))]}
	\\
	=&~\EE_{(x_h,a_h)\sim d^D_h,x_{h+1}\sim P_h(\cdot\mid x_h,a_h)}[w_h(x_h,a_h)(f_h(x_h,a_h)-r_h-f_{h+1}(x_{h+1},\pi_f(x_{h+1})))]
	\\
	=&~\EE_{(x_h,a_h)\sim d^D_h}[w_h(x_h,a_h)(f_h(x_h,a_h)-R_h(x_h,a_h)-\EE_{x_{h+1}\sim P_h(\cdot\mid x_h,a_h)}[f_{h+1}(x_{h+1},\pi_f(x_{h+1}))])]
	\\
	=&~\EE_{(x_h,a_h)\sim d^D_h}[w_h(x_h,a_h)(f_h(x_h,a_h)- (\Tcal_h f_{h+1})(x_h,a_h)))].
\end{align*}

For any $f\in\Fcal,h\in[H]$, we similarly have
\begin{align*}
	\EE[\Lcal_{\Dcal}(f,w^*,h)]&=~\EE\sbr{\frac{1}{n}\sum_{i=1}^n[w_h^*(x_h^{(i)},a_h^{(i)})(f_h(x_h^{(i)},a_h^{(i)})-r_h^{(i)}-f_{h+1}(x_{h+1}^{(i)},\pi_f(x_{h+1}^{(i)})))]}
	\\
	&=~\EE_{(x_h,a_h)\sim d^D_h,x_{h+1}\sim P_h(\cdot\mid x_h,a_h)}[w_h^*(x_h,a_h)(f_h(x_h,a_h)-r_h-f_{h+1}(x_{h+1},\pi_f(x_{h+1})))]
	\\
	&=~\EE_{(x_h,a_h)\sim d^*_h,x_{h+1}\sim P_h(\cdot\mid x_h,a_h)}[f_h(x_h,a_h)-r_h-f_{h+1}(x_{h+1},\pi_f(x_{h+1}))]
	\\
	&=~\EE[f_h(x_h,a_h)-R_h(x_h,a_h)-f_{h+1}(x_{h+1},a_{h+1})\mid a_{0:h}\sim\pi^*, a_{h+1}\sim \pi_f].
\end{align*}
This completes the proof.
\end{proof}

\subsection{Proof of Theorem~\ref{thm:find_v_star}}
\label{app:proof_find_v_star}
\begin{theorem*}[Sample complexity of identifying $v^*$, restatement of \pref{thm:find_v_star}]
%\label{thm:find_v_star}
Suppose  \pref{assum:realizablity_q}, \pref{assum:realizablity_w}, \pref{assum:bound_q}, \pref{assum:bound_w} hold and the total number of samples $nH$ satisfies \[nH\ge \frac{8C^2H^5\log(2|\Fcal||\Wcal|H/\delta)}{\varepsilon^2}.
\]
Then with probability at least $1-\delta$, running \pref{alg:pess_alg} with $\gapmin=0$ and $\alpha=\varepsilon/(2H)$ guarantees 
\[|V_{\hat f}(x_0)-v^*|\le \varepsilon.
\]
\end{theorem*}

\begin{proof}
From our choice of $n$ and \pref{lem:conc}, with probability at least $1-\delta$, for any $f\in\Fcal,w\in\Wcal,h\in[H]$, we have
\[\abr{\Lcal_{\Dcal}(f,w,h)-\EE[\Lcal_{\Dcal}(f,w,h)]}\le\estat\le\varepsilon/(2H).\]
Throughout the proof, we condition on this high probability event. 

From \pref{lem:trans}, for any $w\in\Wcal,h\in[H]$, we have
\begin{align*}
\EE[\Lcal_{\Dcal}(Q^*,w,h)]&=~\EE_{(x_h,a_h)\sim d^D_h}[w_h(x_h,a_h)(Q^*_h(x_h,a_h)-\Tcal_h Q^*_{h+1}(x_h,a_h)]
\\
&=~\EE_{(x_h,a_h)\sim d^D_h}[w_h(x_h,a_h)\cdot 0]%\tag{Bellman optimality equation}
\\
&=~0.
\end{align*}
Therefore, we further have
\[
\Lcal_\Dcal(Q^*,w,h)\le \EE[\Lcal_\Dcal(Q^*,w,h)]+\estat\le\varepsilon/(2H)=\alpha,
\]
which means $Q^*$ satisfies all the constraints.

Then we show that any value function satisfying all constraints (though it may have large average Bellman errors under some distributions) can not be much more pessimistic than $Q^*$. 

From \pref{lem:conc} and \pref{lem:trans}, we know that for any $f\in\Fcal,h\in[H]$,
\begin{align*}
&~\abr{\Ecal(f,\pi^*,h)}
\\
=&~|\EE[f_h(x_h,a_h)-R_h(x_h,a_h)-f_{h+1}(x_{h+1},a_{h+1})\mid a_{0:h}\sim\pi^*, a_{h+1}\sim \pi_f]|
\\
=&~|\EE[\Lcal_{\Dcal}(f,w^*,h)]|
\\
\le&~ \Lcal_{\Dcal}(f,w^*,h)+\estat
\\
\le&~ \alpha+\estat\le \varepsilon/H.
\end{align*}
Therefore, we have
\begin{align*}
V_f(x_0)&=~f_0(x_0,\pi_f(x_0))
\\
&\ge~f_0(x_0,\pi^*(x_0))
\\
&\ge~\EE[R_0(x_0,a_0)+f_1(x_1,a_1)\mid a_{0}\sim\pi^*,a_{1}\sim\pi_f]- \varepsilon/H \tag{$|\Ecal(f,\pi^*,0)|\le \varepsilon/H$}
\\
&\ge~\EE[R_0(x_0,a_0)\mid a_{0}\sim\pi^*]+\EE[f_1(x_1,a_1)\mid a_{0:1}\sim\pi^*] - \varepsilon/H
\\
&\ge~\EE[R_0(x_0,a_0)\mid a_{0}\sim\pi^*]+\EE[R_1(x_1,a_1)+f_2(x_2,a_2)\mid a_{0:1}\sim\pi^*,a_2\sim \pi_f]-2 \varepsilon/H \tag{$|\Ecal(f,\pi^*,1)|\le \varepsilon/H$}
\\
&\ge~\ldots\\
&\ge~\EE\left[\sum_{h=0}^{H-1}R_h(x_h,a_h)\mid a_{0:H-1}\sim\pi^*\right]-H\times  \varepsilon/H=V^*_0(x_0)- \varepsilon.
\end{align*}
Combining the two arguments above, we know that the pessimistic value function $\hat f$ found by the algorithm satisfies \[v^*-\varepsilon=V^*_0(x_0)-\varepsilon  \le V_{\hat f}(x_0) \le V^*_0(x_0)=v^*,\]
where the second inequality is due to pessimism. This completes the proof.
\end{proof}

\subsection{Proof of Theorem~\ref{thm:main}}
\label{app:proof_main}
\begin{theorem*}[Sample complexity of learning a near-optimal policy, restatement of \pref{thm:main}]
%\label{thm:main}
	Suppose  \pref{assum:realizablity_q}, \pref{assum:realizablity_w}, \pref{assum:bound_q}, \pref{assum:bound_w}, \pref{assum:gap_plus} hold and the total number of samples $nH$ satisfies 
	\[nH\ge \frac{8C^2H^7\log(2|\Fcal||\Wcal|H/\delta)}{\varepsilon^2 \gapq^2}.
	\]
	Then with probability at least $1-\delta$, running \pref{alg:pess_alg} with $\alpha=\varepsilon\gapq/(2H^2)$ and $\gapmin=\gapq$ guarantees 
	\[
	v^{\pi_{\hat f}} \ge v^*-\varepsilon.
	\]
\end{theorem*}

\begin{proof}	
From our choice of $n$ and \pref{lem:conc}, we know that with probability at least $1-\delta$, for any $f\in\Fcal,w\in\Wcal,h\in[H]$, we have
\[\abr{\Lcal_{\Dcal}(f,w,h)-\EE[\Lcal_{\Dcal}(f,w,h)]}\le\estat\le\varepsilon\gapq/(2H^2).\]
Throughout the proof, we condition on this high probability event. 

From the definition of $\gapq$, we know that prescreening will not eliminate $Q^*$, i.e., $Q^*\in\Fcal(\gapq)$. Then similar as the proof of \pref{thm:find_v_star}, we have 
\[
\Lcal_\Dcal(Q^*,w,h)\le \EE[\Lcal_\Dcal(Q^*,w,h)]+\estat=\estat\le\varepsilon\gapq/(2H^2)=\alpha,
\]
which means that $Q^*$ satisfies all the constraints.

For any $f\in\Fcal(\gapq)$ that satisfies all the constraints and any $h\in[H]$, we have
\begin{align*}
&~\Ecal(f,\pi^*,h)
\\
=&~|\EE[f_h(x_h,a_h)-R_h(x_h,a_h)-f_{h+1}(x_{h+1},a_{h+1})\mid a_{0:h}\sim\pi^*, a_{h+1}\sim \pi_f]|
\\
=&~|\EE[\Lcal_{\Dcal}(f,w^*,h)]|
\\
\le&~ \Lcal_{\Dcal}(f,w^*,h)+\estat
\\
\le&~\alpha+\estat
\\
\le&~\varepsilon\gapq/H^2.%\le \varepsilon\gapmin/H^2.
\end{align*}
Now we have the following stronger result compared with the proof of \pref{thm:find_v_star}
\begin{align*}
	&~V_f(x_0)\\
	=&~f_0(x_0,\pi_f(x_0))
	\\
	\ge&~f_0(x_0,\pi^*(x_0)) + \gapq \one\{\pi_f(x_0)\neq \pi^*(x_0)\}
	\\
	\ge&~\EE[R_0(x_0,a_0)+f_1(x_1,a_1)\mid a_{0}\sim\pi^*,a_{1}\sim\pi_f]\\
	&\quad + \gapq \one\{\pi_f(x_0)\neq \pi^*(x_0)\}-\varepsilon \gapq/H^2 \tag{$|\Ecal(f,\pi^*,0)|\le \varepsilon \gapq/H^2$}
	\\
	\ge&~\EE[R_0(x_0,a_0)\mid a_{0}\sim\pi^*]+\EE[f_1(x_1,\pi^*(x_1))+\gapq \one\{\pi_f(x_1)\neq \pi^*(x_1)\}\mid a_{0}\sim\pi^*] \\
	&\quad + \gapq \one\{\pi_f(x_0)\neq \pi^*(x_0)\}-\varepsilon \gapq/H^2
	\\
	=&~\EE[R_0(x_0,a_0)\mid a_{0}\sim\pi^*]+\EE[f_1(x_1,a_1)\mid a_{0:1}\sim\pi^*]+\gapq\EE[\one\{\pi_f(x_1)\neq \pi^*(x_1)\}\mid a_{0}\sim\pi^*] \\
	&\quad + \gapq \one\{\pi_f(x_0)\neq \pi^*(x_0)\}-\varepsilon \gapq/H^2
	\\
	\ge&~\EE[R_0(x_0,a_0)\mid a_{0}\sim\pi^*]+\EE[R_1(x_1,a_1)+f_2(x_2,a_2)\mid a_{0:1}\sim\pi^*,a_2\sim \pi_f]\\
	& \quad  + \gapq [\one\{\pi_f(x_0)\neq \pi^*(x_0)\}+\EE[\one\{\pi_f(x_1)\neq \pi^*(x_1)\}\mid a_{0}\sim \pi^*\}]]
	\\
	& \quad-2\varepsilon \gapq/H^2 \tag{$|\Ecal(f,\pi^*,1)|\le \varepsilon \gapq/H^2$}
	\\
	\ge&~\ldots\\
	\ge&~\EE\left[\sum_{h=0}^{H-1}R_h(x_h,a_h)\mid a_{0:H-1}\sim\pi^*\right]+\gapq \EE\left[\sum_{h=0}^{H-1}\one\{\pi_f(x_h)\neq \pi^*(x_h)\}\mid a_{0:H-1}\sim\pi^*\right]\\
	&\quad-H\times \varepsilon \gapq/H^2\\
	=&~V^*_0(x_0)+\gapq \EE\left[\sum_{h=0}^{H-1}\one\{\pi_f(x_h)\neq \pi^*(x_h)\}\mid a_{0:H-1}\sim\pi^*\right]- \varepsilon \gapq/H.
\end{align*}
This implies the pessimistic value function $\hat f$ found by the \pref{alg:pess_alg} satisfies
\[
	V^*_0(x_0) \ge V_{\hat f}(x_0)\ge V^*_0(x_0)+\gapq \EE\left[\sum_{h=0}^{H-1}\one\{\pi_{\hat f}(x_h)\neq \pi^*(x_h)\}\mid a_{0:H-1}\sim\pi^*\right]- \varepsilon \gapq/H
\]
and thus
\begin{equation}
\label{eq:sl_error}
\EE\left[\sum_{h=0}^{H-1}\one\{\pi_{\hat f}(x_h)\neq \pi^*(x_h)\}\mid a_{0:H-1}\sim\pi^*\right]\le \varepsilon/H.
\end{equation}

On the other hand, define each trajectory $\tau$ as $(x_0,a_0,r_0,\ldots,x_{H-1},a_{H-1},r_{H-1},x_H)$, the return of $\tau$ as $\mathrm{Return}(\tau)=r_0+\ldots+r_{H-1}$, and the probability of $\tau$ under policy $\pi$ (i.e., $a_h=\pi(x_h),\forall h\in[H]$) as $\Pr\nolimits_\pi(\tau)$. For any  $f\in\Fcal$, we can decompose the entire trajectory space into three disjoint sets $\Ccal_1=\{\tau=(x_0,a_0,r_0,\ldots,x_{H-1},a_{H-1},r_{H-1},x_H):\forall h\in[H],a_h=\pi^*(x_h)=\pi_f(x_h)\}$,  $\Ccal_2=\{\tau=(x_0,a_0,r_0,\ldots,x_{H-1},a_{H-1},r_{H-1},x_H):\forall h\in[H],a_h=\pi^*(x_h),\exists h\in[H],\pi_f(x_h)\neq\pi^*(x_h)\}$, $\Ccal_3=(\Ccal_1\bigcup\Ccal_2)^\complement$.

Then we calculate $V^{\pi^*}$ and $V^{\pi_f}$ with the definition of these three sets 
\begin{align*}
	V^{\pi^*}_0(x_0)&=~\sum_{\tau\in\Ccal_1\bigcup\Ccal_2\bigcup\Ccal_3}\Pr\nolimits_{\pi^*}(\tau)\text{Return}(\tau)
	\\
	&=~\sum_{\tau\in\Ccal_1}\Pr\nolimits_{\pi^*}(\tau)\text{Return}(\tau)+\sum_{\tau\in\Ccal_2}\Pr\nolimits_{\pi^*}(\tau)\text{Return}(\tau) \tag{Because  $\pi^*$ is greedy policy, any trajectory $\tau\in\Ccal_3$ has 0 probability}
	\\
	&=~\sum_{\tau\in\Ccal_1}\Pr\nolimits_{\pi_f}(\tau)\text{Return}(\tau)+\sum_{\tau\in\Ccal_2}\Pr\nolimits_{\pi^*}(\tau)\text{Return}(\tau) \tag{Definition of $\Ccal_1$}
	\\
	&\le~\sum_{\tau\in\Ccal_1}\Pr\nolimits_{\pi_f}(\tau)\text{Return}(\tau) +\sum_{\tau\in\Ccal_2}\Pr\nolimits_{\pi^*}(\tau)H \tag{$\text{Return}(\tau)\le H$}
	\\
	&\le~\sum_{\tau\in\Ccal_1\bigcup\Ccal_2\bigcup\Ccal_3}\Pr\nolimits_{\pi_f}(\tau)\text{Return}(\tau) +\sum_{\tau\in\Ccal_2}\Pr\nolimits_{\pi^*}(\tau)H \tag{$\text{Return}(\tau)\ge 0$}
	\\
	&=~V^{\pi_f}_0(x_0) +\sum_{\tau\in\Ccal_2}\Pr\nolimits_{\pi^*}(\tau)H.
\end{align*}
It remains to show that $\Pr\nolimits_{\pi^*}(\Ccal_2)=\sum_{\tau\in\Ccal_2}\Pr\nolimits_{\pi^*}(\tau)$ is small. From the definition, any trajectory $\tau=(x_0,a_0,r_0,\ldots,x_{H-1},a_{H-1},r_{H-1},x_H)\in \Ccal_2$ satisfies that $\forall h\in[H], a_h=\pi^*(x_h)$ and $\exists h\in[H], a_h\neq \pi_f(x_h)$. Then for any $\tau\in\Ccal_2$, we can find a unique index $h'\in[H]$ such that $a_0=\pi^*(x_0)=\pi_f(x_0),\ldots,a_{h'-1}=\pi^*(x_{h'-1})=\pi_f(x_{h'-1})$, $a_{h'}= \pi^*(x_{h'})\neq \pi_f(x_{h'})$ (i.e., $h'$ is the smallest index that $\pi_f$ differs from $\pi^*$ in trajectory $\tau$). This implies that $\Ccal_2\subseteq \bigcup_{h'=0}^{H-1}\Ccal_2^{h'}$, where $\Ccal^{h'}_2=\{\tau=(x_0,a_0,r_0,\ldots,x_{H-1},a_{H-1},r_{H-1},x_H):a_0=\pi^*(x_0)=\pi_f(x_0),\ldots,a_{h'-1}=\pi^*(x_{h'-1})=\pi_f(x_{h'-1})$, $a_{h'}= \pi^*(x_{h'})\neq \pi_f(x_{h'})\}$. Since $\EE[\one\{\pi_f(x_{h'})\neq \pi^*(x_{h'})\mid a_{0:h'-1}\sim\pi^*\}] = \Pr\nolimits_{\pi^*}(\Ccal_2^{h'})$, we have 
\begin{align*}
\sum_{\tau\in\Ccal_2}\Pr\nolimits_{\pi^*}(\tau)\le\sum_{h'=0}^{H-1}\sum_{\tau\in\Ccal_2^{h'}}\Pr\nolimits_{\pi^*}(\tau)&=~\EE\left[\sum_{h=0}^{H-1}\one\{\pi_{\hat f}(x_h)\neq \pi^*(x_h)\}\mid a_{0:h-1}\sim\pi^*\right]\\
&=~\EE\left[\sum_{h=0}^{H-1}\one\{\pi_{\hat f}(x_h)\neq \pi^*(x_h)\}\mid a_{0:H-1}\sim\pi^*\right].
\end{align*}
Finally, combining all the results above gives us
\begin{align}
\label{eq:final_err}
	V^{\pi_{\hat f}}_0(x_0)&~\ge V^*_0(x_0)-\sum_{\tau\in\Ccal_2}\Pr\nolimits_{\pi^*}(\tau)H \notag
	\\
	&~\ge V^*_0(x_0)-H\EE\left[\sum_{h=0}^{H-1}\one\{\pi_{\hat f}(x_h)\neq \pi^*(x_h)\}\mid a_{0:H-1}\sim\pi^*\right] \notag
	\\
	&~\ge v^* - H \times \varepsilon/H =v^* - \varepsilon.
\end{align}
This completes the proof.
\end{proof}

\paragraph{Remark}
We notice that \Cref{eq:sl_error} is the error of supervised learning (SL) with 0/1 loss. Therefore, we can directly use the RL to SL reduction in imitation learning literature (e.g., Theorem 2.1 in \citet{ross2010efficient}) to translate it to the final performance difference. It gives us the same as our result in \Cref{eq:final_err}. This second part of the proof is different from the one in \citet{ross2010efficient} and is potentially easier to understand. We believe that it is also of its independent interest.

\section{Proof of Robustness Results}
In this section, we provide the complete proof of misspecificed cases in \pref{sec:appx_error}. We start with some helper lemmas in \pref{app:helper_lemma_approx}. Then we show the proof of \pref{thm:find_v_star_appx} in \pref{app:proof_find_v_star_approx} and the proof of \pref{thm:main_appx} in \pref{app:proof_main_approx}.

\subsection{Helper Lemmas}
\label{app:helper_lemma_approx}
\begin{lemma}[Population loss bound for approximately realizable $\Wcal$]
\label{lem:appx_w}
Recall that the definitions of $\varepsilon_\Wcal$ and $\tilde w^*$ are
\begin{align*}
\varepsilon_{\Wcal}=~\min_{w\in\Wcal}\max_{f\in\Fcal}\max_{h\in[H]}\abr{\EE_{d^D_h}[w_h\cdot (f_h-\Tcal_h f_{h+1})] -\EE_{d^*_h}[f_h-\Tcal_h f_{h+1}]}
\end{align*}
and 
\begin{align*}
\tilde w^*=~\argmin_{w\in\Wcal}\max_{f\in\Fcal}\max_{h\in[H]}\abr{\EE_{d^D_h}[w_h\cdot (f_h-\Tcal_h f_{h+1})] -\EE_{d^*_h}[f_h-\Tcal_h f_{h+1}]}.
\end{align*}
For any $f\in\Fcal,h\in[H]$, we have
\textbf{\begin{align*}
\abr{\Ecal(f,\pi^*,h)} \le \abr{\EE[\Lcal_{\Dcal}(f,\tilde w^*,h)]} + \varepsilon_\Wcal,
\end{align*}}
where $\Ecal(\cdot)$ is the Q-type average Bellman error \citep{jin2021bellman,du2021bilinear} \[\Ecal(f,\pi,h)=\EE[f_h(x_h,a_h)-R_h(x_h,a_h)-f_{h+1}(x_{h+1},a_{h+1})\mid a_{0:h}\sim\pi, a_{h+1}\sim \pi_f]. 	\]
\end{lemma}
\begin{proof}
For any $f\in\Fcal,h\in[H]$, we have
\textbf{\begin{align*}
&~\abr{\Ecal(f,\pi^*,h)}
\\
=&~\EE[f_h(x_h,a_h)-R_h(x_h,a_h)-f_{h+1}(x_{h+1},a_{h+1})\mid a_{0:h}\sim\pi^*, a_{h+1}\sim \pi_f]. 
\\
=&~\abr{\EE_{(x_h,a_h)\sim d^*_h,x_{h+1}\sim P_h(\cdot\mid x_h,a_h)}[f_h(x_h,a_h)-R_h- f_{h+1}(x_{h+1},\pi_f(x_{h+1}))]}
	\\
=&~\abr{\EE_{(x_h,a_h)\sim d^*_h}[f_h(x_h,a_h)-(\Tcal_h f_{h+1})(x_h,a_h)]}
\\
=&~\abr{\EE_{d^*_h}[f_h-\Tcal_h f_{h+1}]}
\\
%=&~\abr{\EE_{(x_h,a_h)\sim d^D_h}[w_h^*(x_h,a_h)(f(x_h,a_h)-R_h(x_h,a_h)- \EE_{x_{h+1}\sim P_h(\cdot\mid x_h,a_h)}f(x_{h+1},\pi_f(x_{h+1})))]}
%\\
%=&~\abr{\EE_{(x_h,a_h)\sim d^D_h}[w_h^*(x_h,a_h)(f(x_h,a_h)-(\Tcal_h f_{h+1})(x_h,a_h))]}
%\\
\le&~\abr{\EE_{d^D_h}[\tilde w_h^*(f_h-\Tcal_h f_{h+1}]}+\abr{\EE_{d^D_h}[\tilde w_h^*\cdot (f_h-\Tcal_h f_{h+1})] -\EE_{d^*_h}[f_h-\Tcal_h f_{h+1}]} 
\\
\le&~\abr{\EE[\Lcal_{\Dcal}(f,\tilde w^*,h)]} + \varepsilon_\Wcal,
\end{align*}}
which completes the proof.
\end{proof}

\begin{lemma}[$\varepsilon_\Fcal$ is weaker than $\ell_\infty$ approximation error] 
\label{lem:appx_f_vs_infty}
Recall that the definitions of $\varepsilon_\Fcal$ and $\tilde Q_{\Fcal}^*$ are
\[
\varepsilon_{\Fcal}=\min_{f\in\Fcal}\max_{w\in\Wcal}\max_{h\in[H]}\left(\abr{\EE_{d^D_h}[w_h\cdot (f_h -\Tcal_h f_{h+1})]}+\abr{f_0(x_0,\pi_f(x_0))-Q^*_0(x_0,\pi^*(x_0))}\right)
\]
and 
\[
\tilde Q_{\Fcal}^*=\argmin_{f\in\Fcal}\max_{w\in\Wcal}\max_{h\in[H]}\left(\abr{\EE_{d^D_h}[w_h\cdot (f_h -\Tcal_h f_{h+1})]}+\abr{f_0(x_0,\pi_f(x_0))-Q^*_0(x_0,\pi^*(x_0))}\right).
\]
Suppose additionally we have mild regularity assumptions on $\Wcal$, i.e., for any $w\in\Wcal,h\in[H]$, $\EE_{d^D_h}[w_h] = 1$ and $w_h\in (\Xcal\times\Acal\rightarrow[0,\infty))$.
Then we have
\[
\varepsilon_\Fcal \le 3 \min_{f\in\Fcal}\max_{h\in[H]}\|f_h-Q^*_h\|_\infty.
\]
\end{lemma}

\begin{proof}
For any $f\in\Fcal,w\in\Wcal,h\in[H]$, we have the following
\begin{align}
\label{eq:proof_loss_to_inf_tran}
&\abr{\EE_{d^D_h}[w_h\cdot (f_h -\Tcal_h f_{h+1})]}\notag\\
\le&~ \abr{\EE_{d^D_h}[w_h\cdot (f_h-Q^*_h -\Tcal_h f_{h+1}+\Tcal_h Q^*_{h+1})]}+\abr{\EE_{d^D_h}[w_h\cdot (Q^*_h -\Tcal_h Q^*_{h+1})]}
\notag\\
\le&~\abr{\EE_{d^D_h}[w_h\cdot (f_h-Q^*_h)]}+\abr{\EE_{d^D_h}[w_h\cdot (\Tcal_h f_{h+1}-\Tcal_h Q^*_{h+1})]}+0\notag\\
\le&~\EE_{d^D_h}[w_h\cdot\|f_h-Q^*_h\|_\infty]+\abr{\EE_{(x_h,a_h)\sim d^D_h,x_{h+1}\sim P_h(\cdot\mid x_h,a_h)}[w_h\cdot (f_{h+1}(x_{h+1},\pi_f(x_{h+1})- Q^*(x_{h+1},\pi^*(x_{h+1})))]}
\notag\\
\le&~ \|f_h-Q^*_h\|_\infty + \EE_{(x_h,a_h)\sim d^D_h\cdot w_h,x_{h+1}\sim P_h(\cdot\mid x_h,a_h)}[|f(x_{h+1},\pi_f(x_{h+1})- Q^*_{h+1}(x_{h+1},\pi^*(x_{h+1}))|],
%\\
%\le&~  \|f_h-Q^*_h\|_\infty + \EE_{(x_h,a_h)\sim d^D_h,x_{h+1}\sim P_h(\cdot\mid x_h,a_h)}[|f_{h+1}(x_{h+1},\pi_f(x_{h+1})- Q^*_{h+1}(x_{h+1},\pi^*(x_{h+1}))|]
\end{align}
where the last inequality is due to the $\EE_{d_h^D}[w_h]=1$ and $w_h\ge 0$.

Now, we bound the second term in \Cref{eq:proof_loss_to_inf_tran}. Using $\varepsilon'$ to denote $\max_{h\in[H]}\|f_h-Q^*_h\|_\infty$, we have
\begin{align*}
&~Q^*_{h+1}(x_{h+1},\pi^*(x_{h+1}))- \varepsilon'\le f_{h+1}(x_{h+1},\pi^*(x_{h+1}))\\
\le&~ f_{h+1}(x_{h+1},\pi_f(x_{h+1})) \le Q^*_{h+1}(x_{h+1},\pi_f(x_{h+1})) + \varepsilon'\le  Q^*_{h+1}(x_{h+1},\pi^*(x_{h+1}))+ \varepsilon'.
\end{align*}
This implies that
\begin{align*}
|f_{h+1}(x_{h+1},\pi_f(x_{h+1})- Q^*_{h+1}(x_{h+1},\pi^*(x_{h+1}))|\le \varepsilon'=\max_{h\in[H]}\|f_h-Q^*_h\|_\infty.
\end{align*}
Therefore, we have
\begin{align*}
&\abr{\EE_{d^D_h}[w_h\cdot (f_h -\Tcal_h f_{h+1})]}\le  \|f_h-Q^*_h\|_\infty+\EE_{(x_h,a_h)\sim d^D_h\cdot w_h,x_{h+1}\sim P_h(\cdot\mid x_h,a_h)}[ \|f_{h+1}-Q^*_{h+1}\|_\infty].
\end{align*}
Since $\EE_{d_h^D}[w_h]=1$, we know that $\EE_{(x_h,a_h)\sim d^D_h\cdot w_h,x_{h+1}\sim P_h(\cdot\mid x_h,a_h)}[\cdot]$ is a probability distribution over $x_{h+1}$. This implies that 
\begin{align*}
&\abr{\EE_{d^D_h}[w_h\cdot (f_h -\Tcal_h f_{h+1})]}\le 2 \max_{h\in[H]}\|f_h-Q^*_h\|_\infty.
\end{align*}
Similarly, we have $\abr{f_0(x_0,\pi_f(x_0))-Q^*_0(x_0,\pi^*(x_0))}\le \max_{h\in[H]}\|f_h-Q^*_h\|_\infty$, thus
\[
\abr{\EE_{d^D_h}[w_h\cdot (f_h -\Tcal_h f_{h+1})]}+\abr{f_0(x_0,\pi_f(x_0))-Q^*_0(x_0,\pi^*(x_0))}\le 3\max_{h\in[H]}\|f_h-Q^*_h\|_\infty.
\]
Taking $\max$ over $h\in[h],w\in\Wcal$ and then taking $\min$ over $f\in\Fcal$ on both sides completes the proof.
\end{proof}

\subsection{Proof of Theorem~\ref{thm:find_v_star_appx}}
%\subsection{Robustness Result for Estimating the Optimal Expected Return}
\label{app:proof_find_v_star_approx}
\begin{theorem*}[Robust version of \pref{thm:find_v_star}, Restatement of \pref{thm:find_v_star_appx}]
%\label{thm:find_v_star_appx}
Suppose  \pref{assum:bound_q}, \pref{assum:bound_w} hold and the total number of samples $nH$ satisfies 
\[nH\ge \frac{8C^2H^5\log(2|\Fcal||\Wcal|H/\delta)}{\varepsilon ^2}.
\]
Then with probability $1-\delta$, running \pref{alg:pess_alg} with $\alpha=\varepsilon /(2H)+\varepsilon_{\Fcal}$ and $\gapmin=0$ guarantees 
\[|V_{\hat f}(x_0)-v^*| \le \varepsilon + H\varepsilon_{\Fcal}+H\varepsilon_\Wcal.
\]
\end{theorem*}

\begin{proof}
From \pref{lem:conc} and our choice $n\ge \frac{8C^2H^4\log(2|\Fcal||\Wcal|H/\delta)}{\varepsilon^2}$, with probability at least $1-\delta$, for any $f\in\Fcal,w\in\Wcal,h\in[H]$, we have
\[
\abr{\Lcal_{\Dcal}(f,w,h)-\EE[\Lcal_{\Dcal}(f,w,h)]}\le\estat\le \varepsilon /(2H).
\]
Throughout the proof, we will condition on this high probability event.

From \pref{lem:trans}, we have
\begin{align*}
|\EE[\Lcal_{\Dcal}(\tilde Q^*_{\Fcal},w,h)]|&=~\abr{\EE_{(x_h,a_h)\sim d^D_h}[w_h(x_h,a_h)(\tilde Q^*_{\Fcal,h}(x_h,a_h)-(\Tcal_h \tilde Q^*_{\Fcal,h+1})(x_h,a_h))]}
\\
&\le~\abr{\EE_{(x_h,a_h)\sim d^D_h}[w_h(x_h,a_h)(\tilde Q^*_{\Fcal,h}(x_h,a_h)-(\Tcal_h \tilde Q^*_{\Fcal,h+1})(x_h,a_h))]}\\
&\quad +\abr{\tilde Q^*_{\Fcal,0}(x_0,\pi_{\tilde Q^*_{\Fcal}}(x_0))-Q^*_0(x_0,\pi^*(x_0))}\\
&\le~ \varepsilon_{\Fcal}.
\end{align*}
When using the relaxed constraints by setting $\alpha=\varepsilon/(2H)+\varepsilon_{\Fcal}$, we can incorporate the approximation errors. More specifically, we have
\begin{align*}
\abr{\Lcal_\Dcal(\tilde Q^*_{\Fcal},w,h)}\le \abr{\EE[\Lcal_\Dcal(\tilde Q^*_{\Fcal},w,h)]}+\estat\le \varepsilon_{\Fcal} +\estat \le \varepsilon/(2H)+\varepsilon_{\Fcal}= \alpha,
\end{align*}
which implies that $\tilde Q^*_{\Fcal}$ will satisfy all constraints.

In addition, for any $f\in\Fcal$ that satisfies all constraints, we have that for any $w\in\Wcal,h\in[H]$,
\[
|\EE[\Lcal_{\Dcal}(f,w,h)]|\le\Lcal_{\Dcal}(f,w,h)+\estat\le \alpha + \estat=\varepsilon /H+\varepsilon_{\Fcal}.%+\varepsilon_\Wcal:=\varepsilon'.
\]
From \pref{lem:appx_w}, we further have
\[
|\Ecal(f,\pi^*,h)|\le\abr{\EE[\Lcal_{\Dcal}(f,\tilde w^*,h)]} + \varepsilon_\Wcal.
\]
Since $\tilde w^*\in\Wcal$, we get
\[
|\Ecal(f,\pi^*,h)|\le\abr{\EE[\Lcal_{\Dcal}(f,\tilde w^*,h)]} + \varepsilon_\Wcal\le\varepsilon /H+\varepsilon_{\Fcal}+\varepsilon_\Wcal:=\varepsilon'.
\]
Following telescoping step in the proof of \pref{thm:find_v_star}, for any $f\in\Fcal,h\in[H]$ that satisfies all constraints, we have
\begin{align*}
V_f(x_0)=f_0(x_0,\pi_f(x_0))\ge V^*_0(x_0)- H\varepsilon'.
\end{align*}
Therefore, we have 
\begin{align*}
V^*_0(x_0)+\varepsilon_{\Fcal}=Q^*_0(x_0,\pi^*(x_0)) + \varepsilon_{\Fcal}\ge \tilde Q^*_0(x_0,\pi_{\tilde Q^*}(x_0))\ge\hat f_0(x_0,\pi_{\hat f}(x_0))\ge V^*_0(x_0)- H\varepsilon',
\end{align*}
where the first inequality is due to the definition of approximation error $\varepsilon_{\Fcal}$ and the second inequality is due to pessimism.
This gives us
\[
|V_{\hat f}(x_0)-v^*| \le \max\{H\varepsilon', \varepsilon_{\Fcal}\}\le \varepsilon + H\varepsilon_{\Fcal}+H\varepsilon_\Wcal,
\]
which completes the proof.
\end{proof}

\subsection{Proof of Theorem~\ref{thm:main_appx}}
%\subsection{Robustness Result for Learning a Near-Optimal policy}
\label{app:proof_main_approx}
\begin{theorem*}[Robust version of \pref{thm:main}, restatement of \pref{thm:main_appx}]
%\label{thm:main_appx}
Suppose  \pref{assum:bound_q}, \pref{assum:bound_w} hold and the total number of samples $nH$ satisfies \[nH\ge \frac{8C^2H^7\log(2|\Fcal||\Wcal|H/\delta)}{\varepsilon^2 \gapmin^2}.\] Then with probability $1-\delta$, running \pref{alg:pess_alg} with a user-specified $\gapmin$ and $\alpha=\varepsilon \gapmin/(2H^2)+\varepsilon_{\Fcal(\gapmin)}$ guarantees \[v^{\pi_{\hat f}} \ge v^*-\varepsilon  - \frac{(H^2+H)\varepsilon_{\Fcal(\gapmin)}+H^2\varepsilon_\Wcal}{\gapmin}.\]
\end{theorem*}

\begin{proof}
From \pref{lem:conc} and our choice $n\ge \frac{8C^2H^6\log(2|\Fcal||\Wcal|H/\delta)}{\varepsilon^2 \gapmin^2}$, with probability at least $1-\delta$, for any $f\in\Fcal,w\in\Wcal,h\in[H]$, we have
\[
\abr{\Lcal_{\Dcal}(f,w,h)-\EE[\Lcal_{\Dcal}(f,w,h)]}\le\estat\le \varepsilon \gapmin/(2H^2).
\]
Throughout the proof, we will condition on this high probability event.

From \pref{lem:trans}, we have
\begin{align*}
|\EE[\Lcal_{\Dcal}(\tilde Q^*_{\Fcal(\gapmin)},w,h)]|&=~\abr{\EE_{(x_h,a_h)\sim d^D_h}[w_h(x_h,a_h)(\tilde Q^*_{\Fcal(\gapmin),h}(x_h,a_h)-(\Tcal_h \tilde Q^*_{\Fcal(\gapmin),h+1})(x_h,a_h))]}
\\
&\le~\abr{\EE_{(x_h,a_h)\sim d^D_h}[w_h(x_h,a_h)(\tilde Q^*_{\Fcal(\gapmin),h}(x_h,a_h)-(\Tcal_h \tilde Q^*_{\Fcal(\gapmin),h+1})(x_h,a_h))]}\\
&\quad +\abr{\tilde Q^*_{\Fcal(\gapmin),0}(x_0,\pi_{\tilde Q^*_{\Fcal(\gapmin)}}(x_0))-Q^*_0(x_0,\pi^*(x_0))}\\
&\le~ \varepsilon_{\Fcal(\gapmin)}.
\end{align*}
When using the relaxed constraints of $\alpha=\varepsilon \gapmin/(2H^2)+\varepsilon_{\Fcal(\gapmin)}$, we can incorporate the approximation errors. More specifically, we have
\begin{align*}
\abr{\Lcal_\Dcal(\tilde Q^*_{\Fcal(\gapmin)},w,h)}&\le~ \abr{\EE[\Lcal_\Dcal(\tilde Q^*_{\Fcal(\gapmin)},w,h)]}+\estat
\\
&\le~ \varepsilon_{\Fcal(\gapmin)} +\estat
\\
&\le~ \varepsilon \gapmin/(2H^2)+\varepsilon_{\Fcal(\gapmin)}= \alpha,
\end{align*}
which implies that $\tilde Q^*_{\Fcal(\gapmin)}$ will satisfy all constraints.

In addition, for any $f\in\Fcal(\gapmin)$ that satisfies all constraints, we have that for any $w\in\Wcal,h\in[H]$,
\[
|\EE[\Lcal_{\Dcal}(f,w,h)]|\le\Lcal_{\Dcal}(f,w,h)+\estat\le \alpha + \estat=\varepsilon \gapmin/H^2+\varepsilon_{\Fcal(\gapmin)}.%+\varepsilon_\Wcal:=\varepsilon'.
\]

From \pref{lem:appx_w}, we further have
\[
|\Ecal(f,\pi^*,h)|\le\abr{\EE[\Lcal_{\Dcal}(f,\tilde w^*,h)]} + \varepsilon_\Wcal.
\]
Since $\tilde w^*\in\Wcal$, we get
\[
|\Ecal(f,\pi^*,h)|\le\abr{\EE[\Lcal_{\Dcal}(f,\tilde w^*,h)]} + \varepsilon_\Wcal\le\varepsilon \gapmin/H^2+\varepsilon_{\Fcal(\gapmin)}+\varepsilon_\Wcal:=\varepsilon'.
\]

Since we run the algorithm on $\Fcal(\gapmin)$, the gap parameter will be $\gapmin$ instead of $\gapq$ in \pref{thm:main}. Following the proof of \pref{thm:main}, for any $f\in\Fcal(\gapmin),h\in[H]$ that satisfies all constraints, we have
\begin{align*}
V_f(x_0)=f_0(x_0,\pi_f(x_0))\ge Q^*_0(x_0,\pi^*(x_0))+\gapmin \EE\left[\sum_{h=0}^{H-1}\one\{\pi_f(x_h)\neq \pi^*(x_h)\}\mid a_{0:H-1}\sim\pi^*\right]- H\varepsilon'.
\end{align*}
Therefore, we have 
\begin{align*}
&~Q^*_0(x_0,\pi^*(x_0)) + \varepsilon_{\Fcal(\gapmin)}
\\
\ge&~ \tilde Q^*_{\Fcal(\gapmin),0}(x_0,\pi_{Q^*_{\Fcal(\gapmin)}}(x_0)) \tag{Definition of approximation error $\varepsilon_{\Fcal(\gapmin)}$}
\\
\ge&~ \hat f_0(x_0,\pi_{\hat f}(x_0)) \tag{Pessimism}
\\
\ge&~ Q^*_0(x_0,\pi^*(x_0))+\gapmin \EE\left[\sum_{h=0}^{H-1}\one\{\pi_f(x_h)\neq \pi^*(x_h)\}\mid a_{0:H-1}\sim\pi^*\right]- H\varepsilon',
\end{align*}
which yields
\[
\EE\left[\sum_{h=0}^{H-1}\one\{\pi_f(x_h)\neq \pi^*(x_h)\}\mid a_{0:H-1}\sim\pi^*\right]\le \rbr{H\varepsilon'  + \varepsilon_{\Fcal(\gapmin)}}/\gapmin.
\]
This translates to the performance difference bound of
\[
V^{\pi_{\hat f}}_0(x_0)\ge v^*-H\rbr{H\varepsilon' + \varepsilon_{\Fcal(\gapmin)}}/\gapmin\ge v^*-\varepsilon - \frac{(H^2+H)\varepsilon_{\Fcal(\gapmin)}+H^2\varepsilon_\Wcal}{\gapmin},
\]
which completes the proof.
\end{proof}

\subsection{Corollary from Theorem \ref{thm:main_appx}}
\label{sec:corr_main_appx}

\pref{thm:main_appx} gives us a convenient way to set the gap parameter $\gapmin$. We show that it can easily handle the case that $\ell_\infty$ approximation error of $\Fcal$ and $\gapq$ are known. We formally define $\ell_\infty$ approximation error and the corresponding best approximator w.r.t. $\Fcal$ as
\[\varepsilon_{\Fcal,\infty}=\min_{f\in\Fcal}\max_{h\in[H]}\|f_h-Q^*_h\|_\infty, \quad \tilde Q^*_{\Fcal,\infty}=\argmin_{f\in\Fcal}\max_{h\in[H]}\|f_h-Q^*_h\|_\infty.
\]
Similarly, we can define the version for $\Fcal(\gapq)$.

Then we have the following corollary.
\begin{corollary}[Corollary from \pref{thm:main_appx}]
\label{corr:main_appx}
Suppose  \pref{assum:bound_q}, \pref{assum:bound_w} hold, the weight function class satisfies the additional mild regularity assumptions stated in \pref{lem:appx_f_vs_infty}. Assume we are given $\varepsilon_{\Fcal,\infty},\gapq$ and $2\varepsilon_{\Fcal,\infty}<\gapq$. If the total number of samples $nH$ satisfies 
\[
nH\ge \frac{8C^2H^7\log(2|\Fcal||\Wcal|H/\delta)}{\varepsilon^2 (\gapq-2\varepsilon_{\Fcal,\infty})^2},
\]
then with probability $1-\delta$, running \pref{alg:pess_alg} with $\gapmin=\gapq-2\varepsilon_{\Fcal,\infty}$ and $\alpha=\varepsilon (\gapq-2\varepsilon_{\Fcal,\infty})/(2H^2)+2\varepsilon_{\Fcal,\infty}$  guarantees 
\[v^{\pi_{\hat f}} \ge v^*-\varepsilon  - \frac{(2H^2+H)\varepsilon_{\Fcal,\infty}+H^2\varepsilon_\Wcal}{\gapq-2\varepsilon_{\Fcal,\infty}}.\]
\end{corollary}

\begin{proof}
From the definition of $\gapq$, $\varepsilon_{\Fcal,\infty}$ and $\tilde Q^*_{\Fcal,\infty}$, we know that 
\[
\gap(\tilde Q^*_{\Fcal,\infty})\ge \gapq-2\varepsilon_{\Fcal,\infty}>0.
\]
Therefore, we have $\tilde Q^*_{\Fcal,\infty}\in \Fcal(\gapq-2\varepsilon_{\Fcal,\infty})$. Together with the definition that $\tilde Q^*_{\Fcal,\infty}$ is the best approximator of $Q^*$ within $\Fcal$ (under $\ell_\infty$ norm), we know that $\tilde Q^*_{\Fcal,\infty}$ is also the best approximator within $\Fcal(\gapq-2\varepsilon_{\Fcal,\infty})$ (under $\ell_\infty$ norm). This implies that \[\varepsilon_{\Fcal(\gapq-2\varepsilon_{\Fcal,\infty}),\infty}=\varepsilon_{\Fcal,\infty}.
\]
In addition, under the mild regularity assumptions stated in \pref{lem:appx_f_vs_infty}, applying \pref{lem:appx_f_vs_infty} tells us
\[
\varepsilon_{\Fcal(\gapq-2\varepsilon_{\Fcal,\infty})} \le 3 \min_{f\in\Fcal(\gapq-2\varepsilon_{\Fcal,\infty})}\max_{h\in[H]}\|f_h-Q^*_h\|_\infty=3\varepsilon_{\Fcal(\gapq-2\varepsilon_{\Fcal,\infty}),\infty}=3\varepsilon_{\Fcal,\infty}.
\]
The remaining part of the proof follows a similar approach as the proof of \pref{thm:main_appx}. Firstly, we have the $1-\delta$ high probability event that for any $f\in\Fcal,w\in\Wcal,h\in[H]$
\[
\abr{\Lcal_{\Dcal}(f,w,h)-\EE[\Lcal_{\Dcal}(f,w,h)]}\le\estat\le \varepsilon (\gapq-2\varepsilon_{\Fcal,\infty})/(2H^2).
\]
Then following the proof \pref{lem:appx_f_vs_infty}, we have
\begin{align*}
|\EE[\Lcal_{\Dcal}(\tilde Q^*_{\Fcal,\infty},w,h)]|&=~\abr{\EE_{ d^D_h}[w_h\cdot (\tilde Q^*_{\Fcal,\infty,h}-\Tcal_h\tilde Q^*_{\Fcal,\infty,h+1})]}
\\
&\le~\abr{\EE_{d^D_h}[w_h\cdot (\tilde Q^*_{\Fcal,\infty,h}-Q^*_h)]}+\abr{\EE_{d^D_h}[w_h\cdot (\Tcal_h \tilde Q^*_{\Fcal,\infty,h+1}-\Tcal_h Q^*_{h+1})]}+0\notag\\
&\le~2\max_{h\in[H]}\|\tilde Q^*_{\Fcal,\infty,h}-Q^*_h\|_\infty=2\varepsilon_{\Fcal,\infty}.
\end{align*}
The empirical loss of $\tilde Q^*_{\Fcal,\infty}$ satisfies
\begin{align*}
\abr{\Lcal_\Dcal(\tilde Q^*_{\Fcal,\infty},w,h)}&\le~ \abr{\EE[\Lcal_\Dcal(\tilde Q^*_{\Fcal,\infty},w,h)]}+\estat
\\
&\le~ \varepsilon (\gapq-2\varepsilon_{\Fcal,\infty})/(2H^2)+2\varepsilon_{\Fcal,\infty}= \alpha,
\end{align*}
which implies that $\tilde Q^*_{\Fcal,\infty}$ will satisfy all constraints.

In addition, for any $f\in\Fcal(\gapq-2\varepsilon_{\Fcal,\infty})$ that satisfies all constraints, we have that for any $w\in\Wcal,h\in[H]$,
\[
|\EE[\Lcal_{\Dcal}(f,w,h)]|\le\Lcal_{\Dcal}(f,w,h)+\estat\le \alpha + \estat=\varepsilon (\gapq-2\varepsilon_{\Fcal,\infty})/H^2+2\varepsilon_{\Fcal,\infty}.%+\varepsilon_\Wcal:=\varepsilon'.
\]
Similarly, we further have
\[
|\Ecal(f,\pi^*,h)|\le\abr{\EE[\Lcal_{\Dcal}(f,\tilde w^*,h)]} + \varepsilon_\Wcal\le\varepsilon (\gapq-2\varepsilon_{\Fcal,\infty})/H^2+2\varepsilon_{\Fcal,\infty}+\varepsilon_\Wcal:=\varepsilon'.
\]
The final performance difference bound is
\[
V^{\pi_{\hat f}}_0(x_0)\ge v^*-H\rbr{H\varepsilon' + \varepsilon_{\Fcal,\infty} }/(\gapq-2\varepsilon_{\Fcal,\infty})\ge v^*-\varepsilon - \frac{(2H^2+H)\varepsilon_{\Fcal,\infty}+H^2\varepsilon_\Wcal}{\gapq-2\varepsilon_{\Fcal,\infty}},
\]
where the difference compared with the derivation in the proof of \pref{thm:main_appx} is that we use $\ell_\infty$ bound to get 
\[
Q_0^*(x_0,\pi^*(x_0))+\varepsilon_{\Fcal,\infty}\ge \tilde Q^*_{\Fcal,\infty,0}(x_0,\pi_{Q^*_{\Fcal,\infty}}(x_0)).
\]
This completes the proof.
\end{proof}

\section{Proof of the Unknown Gap Parameter Setting}
\label{app:proof_main_unknown_gap}
In this section, we present the formal proof of \pref{thm:main_unknown}. We start with a standard helper lemma in \pref{app:helper_lemma_unknown}, which shows the concentration result of Monte Carlo estimate. Then we show the proof of \pref{thm:main_unknown} in \pref{app:proof_main_unknown}.

\subsection{A Helper Lemma}
\label{app:helper_lemma_unknown}
\begin{lemma}[Concentration for Monte Carlo estimate]
\label{lem:conc_mc}
Assume we run policy $\pi$ and collect $m$ trajectories $\cbr{x_0^{(i)},a_0^{(i)},r_0^{(i)},\ldots,x_{H-1}^{(i)},a_{H-1}^{(i)},r_{H-1}^{(i)}}_{i=1}^m$ and our Monte Carlo estimate is defined as 
\[\hat v^\pi:=\frac{1}{m}\sum_{i=1}^m\sum_{h=0}^{H-1} r_h^{(i)}.\]
Then we have 
\[
\abr{\hat v^\pi-v^{\pi}}\le 2H\sqrt{\frac{\log(2/\delta)}{2m}}.
\]
\end{lemma}
\begin{proof}
Define random variable $Y_i:=\sum_{h=0}^{H-1}r_h^{(i)}$. From the definition, we know that $Y_i$ are i.i.d. samples with mean $v^{\pi}$. Applying Hoeffding's inequality and noticing that $|Y_i|\le H$ gives us with probability $1-\delta$,
\[
\abr{\frac{1}{m}\sum_{i=1}^m Y_i-v^{\pi}}\le 2H\sqrt{\frac{\log(2/\delta)}{2m}}.
\]
This completes the proof.
\end{proof}

\subsection{Proof of Theorem~\ref{thm:main_unknown}}
\label{app:proof_main_unknown}
\begin{theorem*}[Sample complexity of finding a near-optimal policy with unknown $\gapq$, restatement of \pref{thm:main_unknown}]
%\label{thm:main_unknown}
Suppose  \pref{assum:realizablity_q}, \pref{assum:realizablity_w}, \pref{assum:bound_q}, \pref{assum:bound_w}, \pref{assum:gap_plus} hold but $\gapq$ is unknown. Assume we have a dataset $\Dcal$ with size $n$ for each $\Dcal_h$ and additional online access to collect
\[(\log(2H/\gapq))^2\cdot \frac{n\log(24/\delta)}{C^2 H}=\tilde O\rbr{\frac{n\log(1/\delta)}{C^2H}}\]
samples. Then with probability at least $1-\delta$, the output policy $\hat \pi$ from \pref{alg:unknown_gap} satisfies
\begin{align*}
%\label{eq:unknonw_q_accu}
v^{\hat \pi}\ge v^* - 5\sqrt{\frac{32C^2H^6\iota(\log(2H/\gapq))}{n\gapq^2}},    
\end{align*}
where $\iota(t)=\log(24|\Fcal||\Wcal|H\cdot 2^t/\delta)$.
\end{theorem*}

\begin{proof}
For \pref{thm:find_v_star}, \pref{thm:main} and Monte Carlo roll out estimate at iteration $t$, we set their high probability event parameter as $\delta'_t:=\delta/(6 \times 2^t)$. Then union bounding over all of them gives us $1-\delta$ high probability event. Our following analysis is conditioned on these high probability events. %From \pref{alg:unknown_gap}, we know that there is at most $\log(H/B_{\mathrm{gap}}) + 1$ iterations.

Firstly, we show that \pref{alg:unknown_gap} will terminate once our guess $\mathrm{gap}^{\mathrm{guess}}_t$ drops below the true $\gapq$. From \pref{thm:find_v_star}, we know that $|\hat v_t^*-v^*|\le \varepsilon_t$. Further, when $\mathrm{gap}^{\mathrm{guess}}_t\le \gapq$, we can guarantee that $Q^*\in\Fcal(\mathrm{gap}^{\mathrm{guess}}_t)$. Therefore, \pref{thm:main} tells us $v^{\hat \pi_t} \ge v^*-\varepsilon_t$. Finally, for Monte Carlo estimate $\hat v^{\hat \pi_t}$, we have $|\hat v^{\hat \pi_t}-v^{\hat \pi_t}|\le \varepsilon_t$. Combining them together yields
\[
\hat v^{\hat \pi_t} \ge v^{\hat \pi_t} - \varepsilon_t \ge v^*-\varepsilon_t-\varepsilon_t\ge \hat v^*_t -\varepsilon_t-\varepsilon_t-\varepsilon_t=\hat v^*_t-3\varepsilon_t,
\]
which means our algorithm will stop in this iteration.

So if we assume the algorithm terminates at iteration $T$, then $T$ satisfies $H/2^T\ge \gapq/2$, thus
\[
T\le \log(2H/\gapq).
\]
Then we prove that the output policy $\hat\pi_T$ satisfies $v^{\hat\pi_T}\ge v^*-5\varepsilon_t$. This can be seen from
\[
v^{\hat \pi_T} \ge \hat v^{\hat \pi_T} -\varepsilon_T \ge \hat v^*_T -3\varepsilon_T -\varepsilon_T \ge  v^* -\varepsilon_T -3\varepsilon_T -\varepsilon_T =v^*-5\varepsilon_T.
\]
Notice that $\varepsilon_t$ will increase as $t$ increases. Therefore, if our algorithm terminates before $\mathrm{gap}^{\mathrm{guess}}_t$ drops below $\gapq$, we will have a better performance guarantee. More specifically, we have \[\varepsilon_T\le\varepsilon_{\log(2H/\gapq)}=\sqrt{\frac{32C^2H^6\iota(\log(2H/\gapq))}{n\gapq^2}}.
\]
Therefore, $\hat\pi_T$ satisfies 
\[v^{\hat \pi_T}\ge v^* - 5\sqrt{\frac{32C^2H^6\iota(\log(2H/\gapq))}{n\gapq^2}},
\]
which has the same order of the accuracy as running \pref{alg:pess_alg} with known $\gapq$ in \pref{thm:main} up to polylog terms.

Finally we calculate the required number of online samples. For iteration $t$, applying \pref{lem:conc_mc}, we require \[H\cdot\frac{2H^2\log(12\times 2^t/\delta)}{\varepsilon_t^2}\le  \frac{2H^3\log(12\times 2^T/\delta)}{\varepsilon_t^2}=\frac{n\log(12\times 2^T/\delta)}{4C^2 H\iota(t) 2^{2t}}\le\frac{n\log(12\times 2^T/\delta)}{C^2 H}\le\frac{nT\log(12\times 2/\delta)}{C^2 H}
\]
samples. Then since we have at most $\log(2H/\gapq)$ iterations, the required number of online samples is at most 
\[
\log(2H/\gapq)\cdot \frac{nT\log(12\times 2/\delta)}{C^2 H}\le(\log(2H/\gapq))^2\cdot \frac{n\log(24/\delta)}{C^2 H}.
\]
This completes the proof.
\end{proof}

\section{Lagrangian Form Algorithm and Results}
\label{app:lang}
In this section, we introduce the Lagrangian form variant of PABC (\pref{alg:pess_alg}) and its sample complexity guarantees. We start with showing its variant PABC-L (\pref{alg:pess_lang_alg}) in \pref{app:alg_lang}. Then we provide the main results of PABC-L in \pref{app:main_lang} and its robustness results in \pref{app:appx_error_lang}.

\subsection{Algorithm}
\label{app:alg_lang}
In this part, we introduce the PABC-L (PABC with Lagrangian form) algorithm as shown in \pref{alg:pess_lang_alg}. Compared with PABC (\pref{alg:pess_alg}), PABC-L does not take the threshold $\alpha$ as input. In addition, it moves the constraints (\Cref{eq:constraint}) to the objective (\Cref{eq:objective}). Furthermore, to estimate $v^*$, it returns  $\hat f_0(x_0, \pi_{\hat f}(x_0))+ H\cdot \max_{w\in\Wcal,h\in[H]} |\Lcal_{\Dcal}(\hat f,w,h)|$ instead of $\hat f_0(x_0, \pi_{\hat f}(x_0))$.
\begin{algorithm}[hbt]
	\caption{PABC-L (PABC with Lagrangian form)}\label{alg:pess_lang_alg}
	\begin{algorithmic}[1]
	    \REQUIRE gap factor $\cgap$, function class $\Fcal$, weight function class $\Wcal$, and dataset $\Dcal$.
	    \STATE Perform prescreening according to input  $\cgap$: %\label{line:prescreen}
	    \begin{align} %\label{line:prescreen}
	    \Fcal(\cgap):=\{f \in \Fcal: \gap(f)\ge\cgap\}.
	    \end{align}
		\STATE Find the pessimism value function in  $\Fcal(\cgap)$ with the Lagrangian form objective %\label{line:pess_select}
		\begin{align}
		\label{eq:objective}
		&\hat f=\argmin_{f\in\Fcal(\cgap)} \rbr{f_0(x_0,\pi_{f}(x_0))+ H\cdot \max_{w\in\Wcal,h\in[H]} |\Lcal_{\Dcal}(f,w,h)|}
    	%\text{s.t.}& \max_{w\in\Wcal,h\in[H]} |\Lcal_{\Dcal}(f,w,h)|\le\alpha,
    	\end{align}
    	where the empirical loss $\Lcal_{\Dcal}(f,w,h)$ is defined as
    	\begin{align}
        \Lcal_{\Dcal}(f,w,h) &%&~=\EE_{\Dcal}[w_h(x_h,a_h)(f_h(x_h,a_h)-r_h-f_{h+1}(x_{h+1},\pi_f(x_{h+1})))].\\
        =\frac{1}{n}\sum_{i=1}^n[w_h(x_h^{(i)},a_h^{(i)})(f_h(x_h^{(i)},a_h^{(i)})-r_h^{(i)}-f_{h+1}(x_{h+1}^{(i)},\pi_f(x_{h+1}^{(i)})))]. %\label{eq:LD}
        \end{align}
		\ENSURE policy $\pi_{\hat f}$ and return estimation $\hat f_0(x_0, \pi_{\hat f}(x_0))+ H\cdot \max_{w\in\Wcal,h\in[H]} |\Lcal_{\Dcal}(\hat f,w,h)|$.
	\end{algorithmic}
\end{algorithm}

\paragraph{Remark} In the objective (\Cref{eq:objective}), we can also use
\begin{align}
\label{eq:objective_another}
&\hat f=\argmin_{f\in\Fcal(\cgap)} \rbr{f_0(x_0,\pi_{f}(x_0))+ \sum_{h=0}^{H-1} \max_{w\in\Wcal} |\Lcal_{\Dcal}(f,w,h)|}.
	%\text{s.t.}& \max_{w\in\Wcal,h\in[H]} |\Lcal_{\Dcal}(f,w,h)|\le\alpha,
\end{align}
From the detailed proofs in the subsequent parts, it is easy to see that the theoretical results hold under this objective (\Cref{eq:objective_another}).

\subsection{Main Guarantees}
\label{app:main_lang}
In this part, we present the main sample complexity results of PABC-L (\pref{alg:pess_lang_alg}). In parallel with \pref{sec:main}, we show that PABC-L can identify $v^*$ without the gap assumption in \pref{app:find_v_star_lang} and show that PABC-L with the gap assumption learns a near-optimal policy in \pref{app:find_near_optimal_lang}. 

\subsubsection{ESTIMATING OPTIMAL EXPECTED RETURN}
\label{app:find_v_star_lang}
We show the sample complexity bound and the proof for PABC-L to identify $v^*$. The bound is the same as that of PABC (\pref{thm:find_v_star}).
\begin{theorem}[Sample complexity of identifying $v^*$, Lagrangian version]
\label{thm:find_v_star_lang}
Suppose  \pref{assum:realizablity_q}, \pref{assum:realizablity_w}, \pref{assum:bound_q}, \pref{assum:bound_w} hold and the total number of samples $nH$ satisfies \[nH\ge \frac{8C^2H^5\log(2|\Fcal||\Wcal|H/\delta)}{\varepsilon^2}.
\]
Then with probability at least $1-\delta$, running \pref{alg:pess_lang_alg} with $\gapmin=0$ guarantees 
\[|V_{\hat f}(x_0)-v^*|\le \varepsilon.
\]
\end{theorem}

\begin{proof}
The proof mostly follows the proof of \pref{thm:find_v_star}, and we only show the different and crucial steps here. We still condition on the high probability event from concentration (\pref{lem:conc}).

From the concentration result and the choice of $n$, we have the bound for $Q^*$:
\[
V^*_0(x_0) + H\cdot \max_{w\in\Wcal,h\in[H]} |\Lcal_{\Dcal}(Q^*,w,h)| \le V^*_0(x_0) + H\estat,
\]
where $\estat\le \varepsilon/H$.

From pessimism and the objective in \pref{alg:pess_lang_alg}, we have 
\begin{align*}
V^*_0(x_0) + H\cdot \max_{w\in\Wcal,h\in[H]} |\Lcal_{\Dcal}(Q^*,w,h)| &\ge~ V_{\hat f}(x_0)+H\cdot \max_{w\in\Wcal,h\in[H]} |\Lcal_{\Dcal}(\hat f,w,h)|.
\end{align*}
Therefore, we get
\begin{align}
\label{eq:find_v_star_lang_1}
V_0^*(x_0)+H\estat \ge V_{\hat f}(x_0)+H\cdot \max_{w\in\Wcal,h\in[H]} |\Lcal_{\Dcal}(\hat f,w,h)|.
\end{align}
For any $f\in\Fcal$, following the telescoping step in the proof of \pref{thm:find_v_star}, we know that
\begin{align*}
V_f(x_0)&=~f_0(x_0,\pi_f(x_0))
\\
&\ge~f_0(x_0,\pi^*(x_0))
\\
&=~\EE[R_0(x_0,a_0)+f_1(x_1,a_1)\mid a_{0}\sim\pi^*,a_{1}\sim\pi_f]+ \Ecal(f,\pi^*,0)
\\
&\ge~\EE[R_0(x_0,a_0)\mid a_{0}\sim\pi^*]+\EE[f_1(x_1,a_1)\mid a_{0:1}\sim\pi^*] + \Ecal(f,\pi^*,0)
\\
&\ge~\EE[R_0(x_0,a_0)\mid a_{0}\sim\pi^*]+\EE[R_1(x_1,a_1)+f_2(x_2,a_2)\mid a_{0:1}\sim\pi^*,a_2\sim \pi_f]+ \Ecal(f,\pi^*,1)+ \Ecal(f,\pi^*,0)
\\
&\ge~\ldots
\\
&\ge~\EE\left[\sum_{h=0}^{H-1}R_h(x_h,a_h)\mid a_{0:H-1}\sim\pi^*\right]+\sum_{h=0}^{H-1}\Ecal(f,\pi^*,h)
\\
&\ge~ V^*_0(x_0)-  \sum_{h=0}^{H-1}|\Ecal(f,\pi^*,h)|.
\end{align*}
Therefore, we get
\begin{align}
\label{eq:find_v_star_lang_2}
&~V_{\hat f}(x_0)+H\cdot \max_{w\in\Wcal,h\in[H]} |\Lcal_{\Dcal}(\hat f,w,h)|
\notag\\
\ge&~ V_0^*(x_0)-  \sum_{h=0}^{H-1}|\Ecal(\hat f,\pi^*,h)|+H\cdot \max_{w\in\Wcal,h\in[H]} |\Lcal_{\Dcal}(\hat f,w,h)|\notag
\\
\ge&~ V_0^*(x_0)-\sum_{h=0}^{H-1}|\Ecal(\hat f,\pi^*,h)|+H\cdot \max_{w\in\Wcal,h\in[H]} |\EE[\Lcal_{\Dcal}(\hat f,w,h)]|-H\estat
\notag\\
\ge &~V_0^*(x_0)-\sum_{h=0}^{H-1}|\Ecal(\hat f,\pi^*,h)|+ \sum_{h=0}^{H-1}|\EE[\Lcal_{\Dcal}(\hat f,w^*,h)]|-H\estat
\notag\\
= &~V_0^*(x_0)-\sum_{h=0}^{H-1}|\Ecal(\hat f,\pi^*,h)|+\sum_{h=0}^{H-1}|\Ecal(\hat f,\pi^*,h)|-H\estat
\notag\\
= &~V_0^*(x_0)-H\estat.
\end{align}
Combining \Cref{eq:find_v_star_lang_1} and \Cref{eq:find_v_star_lang_2} yields
\[
|V_{\hat f}(x_0)+H\cdot \max_{w\in\Wcal,h\in[H]} |\Lcal_{\Dcal}(\hat f,w,h)|-v^*|=|V_{\hat f}(x_0)+H\cdot \max_{w\in\Wcal,h\in[H]} |\Lcal_{\Dcal}(\hat f,w,h)|-V_0^*(x_0)|\le H\estat\le \varepsilon,
\]
which completes the proof.
\end{proof}

\subsubsection{LEARNING A NEAR-OPTIMAL POLICY}
\label{app:find_near_optimal_lang}
Here we present the result for learning a near optimal policy. Compared with its counterpart (\pref{thm:main}), the sample complexity only differs in the constant.
\begin{theorem}[Sample complexity of learning a near-optimal policy, Lagrangian version]
\label{thm:main_lang}
	Suppose  \pref{assum:realizablity_q}, \pref{assum:realizablity_w}, \pref{assum:bound_q}, \pref{assum:bound_w}, \pref{assum:gap_plus} hold and the total number of samples $nH$ satisfies 
	\[nH\ge \frac{32C^2H^7\log(2|\Fcal||\Wcal|H/\delta)}{\varepsilon^2 \gapq^2}.
	\]
	Then with probability at least $1-\delta$, running \pref{alg:pess_lang_alg} with $\gapmin=\gapq$ guarantees 
	\[
	v^{\pi_{\hat f}} \ge v^*-\varepsilon.
	\]
\end{theorem}

\begin{proof}
The proof mostly follows the proof of \pref{thm:main} and \pref{thm:find_v_star_lang}, and we only show the different and crucial steps here. We still condition on the high probability event from concentration (\pref{lem:conc}).

Similar as the proof of \pref{thm:find_v_star_lang}, from pessimism, we have
\begin{align}
\label{eq:main_lang_2}
V_0^*(x_0)+H\estat \ge V_{\hat f}(x_0)+H\cdot \max_{w\in\Wcal,h\in[H]} |\Lcal_{\Dcal}(\hat f,w,h)|,
\end{align}
where $\estat\le \varepsilon\gapq/(2H^2)$.

On the other hand, following the proof of \pref{thm:main} and \pref{thm:find_v_star_lang}, we have
\begin{align}
\label{eq:main_lang_1}
&~V_f(x_0)+H\cdot \max_{w\in\Wcal,h\in[H]} |\Lcal_{\Dcal}(\hat f,w,h)|
\notag\\
\ge&~ V^*_0(x_0)+\gapq \EE\left[\sum_{h=0}^{H-1}\one\{\pi_f(x_h)\neq \pi^*(x_h)\}\mid a_{0:H-1}\sim\pi^*\right]- \sum_{h=0}^{H-1}|\Ecal(\hat f,w^*,h)|+H\cdot \max_{w\in\Wcal,h\in[H]} |\Lcal_{\Dcal}(\hat f,w,h)|
\notag\\
\ge&~V^*_0(x_0)+\gapq \EE\left[\sum_{h=0}^{H-1}\one\{\pi_f(x_h)\neq \pi^*(x_h)\}\mid a_{0:H-1}\sim\pi^*\right]- \sum_{h=0}^{H-1}|\Ecal(\hat f,w^*,h)|+\sum_{h=0}^{H-1}|\Ecal(\hat f,w^*,h)|-H\estat
\notag\\
\ge&~V^*_0(x_0)+\gapq \EE\left[\sum_{h=0}^{H-1}\one\{\pi_f(x_h)\neq \pi^*(x_h)\}\mid a_{0:H-1}\sim\pi^*\right]-H\estat.
\end{align}
Combining \Cref{eq:main_lang_2} and \Cref{eq:main_lang_1} yields
\[\EE\left[\sum_{h=0}^{H-1}\one\{\pi_f(x_h)\neq \pi^*(x_h)\}\mid a_{0:H-1}\sim\pi^*\right]\le 2H\estat /\gapq\le\varepsilon.
\]
The remaining steps are followed from the proof of \pref{thm:main}.
\end{proof}

\subsection{Robustness to Misspecification}
\label{app:appx_error_lang}
In this part, we present the sample complexity results of PABC-L (\pref{alg:pess_lang_alg}) under misspecification. In parallel with \pref{sec:appx_error}, we show that PABC-L can identify $v^*$ in \pref{app:find_v_star_appx_lang}  and show its results for learning a near-optimal policy in \pref{app:find_near_optimal_appx_lang}. The major advantage of PABC-L is that it does not take $\alpha$ as the input, therefore, we no longer require the knowledge of approximation errors. 

\subsubsection{ESTIMATING OPTIMAL EXPECTED RETURN}
\label{app:find_v_star_appx_lang}
We present the result for identifying $v^*$. The sample complexity of PABC-L is the same as its counterpart (\pref{thm:find_v_star_appx}).
\begin{theorem}[Robust version of \pref{thm:find_v_star_lang}]
\label{thm:find_v_star_appx_lang}
Suppose  \pref{assum:bound_q}, \pref{assum:bound_w} hold and the total number of samples $nH$ satisfies 
\[nH\ge \frac{8C^2H^5\log(2|\Fcal||\Wcal|H/\delta)}{\varepsilon ^2}.
\]
Then with probability $1-\delta$, running \pref{alg:pess_lang_alg} with $\gapmin=0$ guarantees 
\[|V_{\hat f}(x_0)-v^*| \le \varepsilon + H\varepsilon_{\Fcal}+H\varepsilon_\Wcal.
\]
\end{theorem}

\begin{proof}
The proof mostly follows the proof of \pref{thm:find_v_star_appx} and \pref{thm:find_v_star_lang}, and we only show the different and crucial steps here. We still condition on the high probability event from concentration (\pref{lem:conc}).

For $\tilde Q^*_{\Fcal}$, from the concentration result and the definition of $\varepsilon_{\Fcal}$, we get
\[
\tilde Q^*_{\Fcal,0}(x_0,\pi_{Q^*_{\Fcal}(x_0)}) + H\cdot \max_{w\in\Wcal,h\in[H]} |\Lcal_{\Dcal}(\tilde Q^*_{\Fcal},w,h)| \le V^*_0(x_0) +H\varepsilon_{\Fcal}+ H\estat,
\]
where $\estat\le \varepsilon/H$.

From pessimism and the objective in  \pref{alg:pess_lang_alg}, we have 
\begin{align*}
\tilde Q^*_{\Fcal,0}(x_0,\pi_{Q^*_{\Fcal}(x_0)}) + H\cdot \max_{w\in\Wcal,h\in[H]} |\Lcal_{\Dcal}(\tilde Q^*_{\Fcal},w,h)| &\ge~ V_{\hat f}(x_0)+H\cdot \max_{w\in\Wcal,h\in[H]} |\Lcal_{\Dcal}(\hat f,w,h)|.
\end{align*}
Therefore, we get
\begin{align}
\label{eq:find_v_star_appx_lang_1}
V_0^*(x_0)+H\varepsilon_{\Fcal}+H\estat \ge V_{\hat f}(x_0)+H\cdot \max_{w\in\Wcal,h\in[H]} |\Lcal_{\Dcal}(\hat f,w,h)|.
\end{align}
For any $f\in\Fcal$, following the telescoping step in the proof of \pref{thm:find_v_star_lang}, we know that
\begin{align*}
    V_f(x_0)\ge V^*_0(x_0)-  \sum_{h=0}^{H-1}|\Ecal(f,\pi^*,h)|.
\end{align*}
Therefore, similar as the proof of \pref{thm:find_v_star_lang} and applying \pref{lem:appx_w}, we get
\begin{align}
\label{eq:find_v_star_appx_lang_2}
&~V_{\hat f}(x_0)+H\cdot \max_{w\in\Wcal,h\in[H]} |\Lcal_{\Dcal}(\hat f,w,h)|
\notag\\
\ge&~ V_0^*(x_0)-\sum_{h=0}^{H-1}|\Ecal(\hat f,\pi^*,h)|+H\cdot \max_{w\in\Wcal,h\in[H]} |\EE[\Lcal_{\Dcal}(\hat f,w,h)]|-H\estat
\notag\\
\ge &~V_0^*(x_0)-\sum_{h=0}^{H-1}|\Ecal(\hat f,\pi^*,h)|+ \sum_{h=0}^{H-1}|\EE[\Lcal_{\Dcal}(\hat f,\tilde w^*,h)]|-H\estat
\notag\\
\ge &~V_0^*(x_0)-\sum_{h=0}^{H-1}|\Ecal(\hat f,\pi^*,h)|+\sum_{h=0}^{H-1}|\Ecal(\hat f,\pi^*,h)|-H\varepsilon_{\Wcal}- H\estat
\notag\\
= &~V_0^*(x_0)-H\varepsilon_{\Wcal}-H\estat.
\end{align}
Combining \Cref{eq:find_v_star_appx_lang_1} and \Cref{eq:find_v_star_appx_lang_2} yields
\begin{align*}
|V_{\hat f}(x_0)+H\cdot \max_{w\in\Wcal,h\in[H]} |\Lcal_{\Dcal}(\hat f,w,h)|-v^*|&=~|V_{\hat f}(x_0)+H\cdot \max_{w\in\Wcal,h\in[H]} |\Lcal_{\Dcal}(\hat f,w,h)|-V_0^*(x_0)|
\\
&=~ H(\varepsilon_{\Fcal}+ \varepsilon_{\Wcal}+\estat)
\\
&\le~  \varepsilon + H(\varepsilon_{\Fcal}+ \varepsilon_{\Wcal}),
\end{align*}
which completes the proof.
\end{proof}

\subsubsection{LEARNING A NEAR-OPTIMAL POLICY}
\label{app:find_near_optimal_appx_lang}
In this part, we show the results for learning a near-optimal policy. Compared with the ones for PABC (\pref{thm:main_appx} and \pref{corr:main_appx}), the differences are only the constants.

\begin{theorem}[Robust version of \pref{thm:main_lang}]
\label{thm:main_appx_lang}
Suppose  \pref{assum:bound_q}, \pref{assum:bound_w} hold and the total number of samples $nH$ satisfies 
\[nH\ge \frac{32C^2H^7\log(2|\Fcal||\Wcal|H/\delta)}{\varepsilon^2 \gapmin^2}.\]
Then with probability $1-\delta$, running \pref{alg:pess_lang_alg} with a user-specified $\gapmin$ guarantees 
\[v^{\pi_{\hat f}} \ge v^*-\varepsilon  - \frac{H^2\varepsilon_{\Fcal(\gapmin)}+H^2\varepsilon_\Wcal}{\gapmin}.\]
\end{theorem}

\begin{proof}
The proof mostly follows the proof of \pref{thm:main_lang} and \pref{thm:find_v_star_appx_lang}, and we only show the different and crucial steps here. We still condition on the high probability event from concentration (\pref{lem:conc}).

Similar as the proof of \pref{thm:find_v_star_appx_lang}, we have
\begin{align}
\label{eq:main_appx_lang_1}
V_0^*(x_0)+H\varepsilon_{\Fcal(\gapmin)}+H\estat \ge V_{\hat f}(x_0)+H\cdot \max_{w\in\Wcal,h\in[H]} |\Lcal_{\Dcal}(\hat f,w,h)|,
\end{align}
where $\estat\le \varepsilon\gapmin/(2H^2)$.

On the other hand, following the proof of \pref{thm:main_lang} and \pref{thm:find_v_star_appx_lang}, we have
\begin{align}
\label{eq:main_appx_lang_2}
&~V_{\hat f}(x_0)+H\cdot \max_{w\in\Wcal,h\in[H]} |\Lcal_{\Dcal}(\hat f,w,h)|
\notag\\
\ge&~ V^*_0(x_0)+\gapmin \EE\left[\sum_{h=0}^{H-1}\one\{\pi_f(x_h)\neq \pi^*(x_h)\}\mid a_{0:H-1}\sim\pi^*\right]- \sum_{h=0}^{H-1}|\Ecal(\hat f,w^*,h)|
\notag\\
&~\quad +H\cdot \max_{w\in\Wcal,h\in[H]} |\Lcal_{\Dcal}(\hat f,w,h)|
\notag\\
\ge&~V^*_0(x_0)+\gapmin \EE\left[\sum_{h=0}^{H-1}\one\{\pi_f(x_h)\neq \pi^*(x_h)\}\mid a_{0:H-1}\sim\pi^*\right]-H\varepsilon_{\Wcal}-H\estat.
\end{align}
Combining \Cref{eq:main_appx_lang_1} and \Cref{eq:main_appx_lang_2} yields
\[\EE\left[\sum_{h=0}^{H-1}\one\{\pi_f(x_h)\neq \pi^*(x_h)\}\mid a_{0:H-1}\sim\pi^*\right]\le H(2\estat+\varepsilon_{\Wcal}+\varepsilon_{\Fcal(\gapmin)}) /\gapmin.
\]
The remaining steps can be followed from the proof of \pref{thm:main}.
\end{proof}

\begin{corollary}[Corollary from \pref{thm:main_appx_lang}]
\label{corr:main_appx_alg}
Suppose  \pref{assum:bound_q}, \pref{assum:bound_w} hold, the weight function class satisfies the additional mild regularity assumptions stated in \pref{lem:appx_f_vs_infty}. Assume we are given $\varepsilon_{\Fcal,\infty},\gapq$ and $2\varepsilon_{\Fcal,\infty}<\gapq$. If the total number of samples $nH$ satisfies 
\[
nH\ge \frac{8C^2H^7\log(2|\Fcal||\Wcal|H/\delta)}{\varepsilon^2 (\gapq-2\varepsilon_{\Fcal,\infty})^2},
\]
then with probability $1-\delta$, running \pref{alg:pess_lang_alg} with $\gapmin=\gapq-2\varepsilon_{\Fcal,\infty}$ guarantees 
\[v^{\pi_{\hat f}} \ge v^*-\varepsilon  - \frac{2H^2\varepsilon_{\Fcal,\infty}+H^2\varepsilon_\Wcal}{\gapq-2\varepsilon_{\Fcal,\infty}}.\]
\end{corollary}

\begin{proof}
The proof mostly follows the proof of \pref{corr:main_appx} and \pref{thm:main_appx_lang}, and we only show the different and crucial steps here. We still condition on the high probability event from concentration (\pref{lem:conc}).

Similar as the proof of \pref{corr:main_appx} and \pref{thm:main_appx_lang}, we have
\begin{align}
\label{eq:corr_appx_lang_1}
V_0^*(x_0)+2H\varepsilon_{\Fcal,\infty}+H\estat \ge V_{\hat f}(x_0)+H\cdot \max_{w\in\Wcal,h\in[H]} |\Lcal_{\Dcal}(\hat f,w,h)|.
\end{align}
On the other hand, following the proof of \pref{thm:main_appx_lang}, we have
\begin{align}
\label{eq:corr_appx_lang_2}
&~V_{\hat f}(x_0)+H\cdot \max_{w\in\Wcal,h\in[H]} |\Lcal_{\Dcal}(\hat f,w,h)|
\notag\\
%\ge&~ V^*_0(x_0)+\gap(\Fcal(\gapmin)) \EE\left[\sum_{h=0}^{H-1}\one\{\pi_f(x_h)\neq \pi^*(x_h)\}\mid a_{0:H-1}\sim\pi^*\right]- \sum_{h=0}^{H-1}|\Ecal(\hat f,w^*,h)|+H\cdot \max_{w\in\Wcal,h\in[H]} |\Lcal_{\Dcal}(\hat f,w,h)|
%\notag\\
\ge&~V^*_0(x_0)+(\gapq-2\varepsilon_{\Fcal,\infty}) \EE\left[\sum_{h=0}^{H-1}\one\{\pi_f(x_h)\neq \pi^*(x_h)\}\mid a_{0:H-1}\sim\pi^*\right]-H\varepsilon_{\Wcal}-H\estat.
\end{align}
Combining \Cref{eq:corr_appx_lang_1} and \Cref{eq:corr_appx_lang_2} yields
\[\EE\left[\sum_{h=0}^{H-1}\one\{\pi_f(x_h)\neq \pi^*(x_h)\}\mid a_{0:H-1}\sim\pi^*\right]\le H(2\estat+\varepsilon_{\Wcal}+2\varepsilon_{\Fcal,\infty}) /(\gapq-2\varepsilon_{\Fcal,\infty}).
\]
The remaining steps can be followed from the proof of \pref{thm:main}.
\end{proof}

\section{Discussion on the Data Coverage Assumption}
\label{app:conc_example}

In this section, we provide an example that shows our data coverage assumption is more relaxed than the $\pi^*$-concentrability assumption in \citet{zhan2022offline} (their Assumption 1) based on raw density ratios. Notice that their assumption translates into $d^*_h(x_h,a_h)/d^D_h(x_h,a_h)\le C,\forall h\in[H],x_h\in\Xcal_h,a_h\in\Acal$ in our finite-horizon episodic setting. We will show an instance where there exists some $h,(x_h,a_h)$ such that $d^*_h(x_h,a_h)/d^D_h(x_h,a_h)=\infty$ and $w^*$ does not even exist (thus $w^*\notin \Wcal$), but we still have $\varepsilon_\Wcal=0$. Therefore, our robust version of sample complexity results can give us meaningful guarantees, however, we cannot apply the (robustness) results in \citet{zhan2022offline}.

\begin{figure}[h!]
	\center
	\begin{tikzpicture}[scale=3]
		\node[state] (s0) at (0,0) {$x_0$};
		%\node[state] (s1) [below left=7em of s0]  {$\mathrm{Null}$};
		\node[state] (s2) [below=4.2em of s0] {$\mathrm{Null}$};
		%\node[state] (s3) [below right=7em of s0] {$\mathrm{Null}$};
		
		%\draw[->] (s0) -> node[near start,above = .2 em,left=.3em] {$\pi^*,\textrm{L}$} (s1);
		\draw[->] (s0) edge[bend right=60] node[near start,above = .2 em,left=.3em] {$\pi^*,\textrm{L}$} (s2);
		\draw[->] (s0) -> node[near start,below = .2 em,right=.em] {$\textrm{M}$} (s2);
		\draw[->] (s0) edge[bend left=60] node[near start,below = .2 em,right=.3em] {$\textrm{R}$} (s2);
		%\draw[->] (s0) -> node[near start,below = .2 em,right=.3em] {$\textrm{R}$} (s3);
	\end{tikzpicture}
	\caption{Example for comparison with $\pi^*$-concentrability assumption \citep{zhan2022offline}.}
	\label{fig:conc_example}
\end{figure}
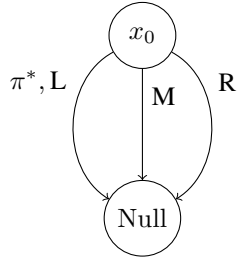

\renewcommand{\arraystretch}{1.25}
\begin{table}[htb]
\begin{center}
	\begin{tabular}{ |c|c|c|c|c| } 
		\hline
		&$(x_0,\mathrm{L})$ & $(x_0,\mathrm{M})$ & $(x_0,\mathrm{R})$  \\ 
		\hline
		$R$ & 0.8 & 0.6 & 0.3  \\ 
		$Q^*$ & 0.8 & 0.6 & 0.3 \\ 
		$f$ & 0.7 & 0.3 & 0.8 \\ 
		\hline 
		$d^*$ & 1 & 0 & 0 \\ 
		$d^D$ & 0  & 0.5 & 0.5 \\
		$w$ & 0 & 1 & 1 \\ 
		\hline
	\end{tabular}
\end{center}
\caption{Example for comparison with $\pi^*$-concentrability assumption \citep{zhan2022offline}.}
\label{table:conc_example}
\end{table}

As shown in \pref{fig:conc_example}, circles denote states and arrows denote actions with deterministic transitions. In this MDP, the length of horizon is $H=1$ and taking any action $\mathrm{L}$, $\mathrm{M}$, or $\mathrm{R}$ at the initial state $x_0$ transits to the $\mathrm{Null}$ terminal state. Since $H=1$, in the following discussion we drop the subscript $h$ for simplicity. In \Cref{table:conc_example}, we show the reward function, the optimal value function $Q^*$, the bad function $f$, the density-ratio function of the optimal policy $d^*$, the data distribution $d^D$, and the weight function $w$. We construct a singleton weight function class $\Wcal=\{w\}$ and a realizable function class $\Fcal=\{Q^*,f\}$. One can easily verify that $d^*(x_0,\mathrm{L})/d^D(x_0,\mathrm{L})=\infty$, $w^*$ does not exist, and the approximation error $\varepsilon_\Wcal$ as defined in \Cref{eq:appx_w} is 0. 

\end{document}